\documentclass[11pt]{article}

\usepackage{amsmath}
\usepackage{graphicx}
\usepackage{natbib}
\usepackage{url} 

\usepackage{comment}

\usepackage{import}
\usepackage{amsmath, amssymb, amsfonts, amscd, xspace, pifont, amsthm, fancyhdr}
\usepackage{xcolor}
\usepackage{physics}
\usepackage{hyperref} 
\usepackage{mathrsfs}
\usepackage{caption} 
\usepackage{subcaption} 
\usepackage{import}
\usepackage{adjustbox}
\usepackage{mathtools}
\usepackage{lipsum}
\usepackage{dirtytalk}
\usepackage{bm}
\usepackage{tablefootnote}
\usepackage{threeparttable}
\usepackage{multirow}
\usepackage{multicol}
\usepackage{mathtools}
\usepackage[ruled,vlined]{algorithm2e}

\theoremstyle{definition}
\newtheorem{theorem}{Theorem}
\newtheorem*{theorem*}{Theorem}
\newtheorem{definition}{Definition}
\newtheorem{proposition}{Proposition}
\newtheorem*{proposition*}{Proposition}
\newtheorem{lemma}{Lemma}
\newtheorem*{lemma*}{Lemma}
\newtheorem{remark}{Remark}

\renewcommand{\tr}{\text{Tr}}

\newcommand{\V}[1]{\boldsymbol{#1}}
\newcommand{\M}[1]{\mathbf{#1}}

\newcommand{\bome}{\boldsymbol{\Omega}}
\newcommand{\bSig}{\boldsymbol{\Sigma}}

\newcommand{\blam}{\boldsymbol{\lambda}}
\newcommand{\bsig}{\boldsymbol{\sigma}}

\newcommand{\balp}{\boldsymbol{\alpha}}
\newcommand{\bic}{\boldsymbol{I}^c}
\newcommand{\bqc}{\boldsymbol{q}^c}

\newcommand{\diag}{\mathrm{diag}}

\newcommand{\Qing}[1]{\textcolor{red}{#1}}
\newcommand{\Ning}[1]{\textcolor{blue}{#1}}

\DeclareMathOperator*{\argmax}{arg\,max}

\def\threeImages#1#2#3#4#5#6#7#8#9 
{
  \centerline{\hfill\makebox[#2]{#3}\hfill\makebox[#5]{#6}\hfill\makebox[#8]{#9}\hfill}
  \centerline{\hfill
    \includegraphics[width=#2]{#1}
    \hfill
    \includegraphics[width=#5]{#4}
    \hfill
    \includegraphics[width=#8]{#7}
    \hfill}
}

\def\twoImages#1#2#3#4#5#6 
{
  \centerline{\hfill\makebox[#2]{#3}\hfill\makebox[#5]{#6}\hfill}
  \centerline{\hfill
    \includegraphics[width=#2]{#1}
    \hfill
    \includegraphics[width=#5]{#4}
    \hfill}
}

\newcommand{\startsupplement}{%
    \setcounter{section}{0}%
    \renewcommand{\thesection}{S.\arabic{section}}%
}

\begin{document}

\title{Community Detection with Heterogeneous Block Covariance Model}

  \author{Xiang Li\\
    Department of Statistics, George Washington University  \\
    and \\
    Yunpeng Zhao \\
    Department of Statistics, Colorado State University \\
    and \\
    Qing Pan \\
    Department of Biostatistics and Bioinformatics, George Washington University \\
    and \\
    Ning Hao \\
    Department of Mathematics, University of Arizona 
    }

\maketitle

\begin{abstract}
Community detection is the task of clustering objects based on their pairwise relationships. Most of the model-based community detection methods, such as the stochastic block model and its variants, are designed for networks with binary (yes/no) edges. In many practical scenarios, edges often possess continuous weights, spanning positive and negative values, which reflect varying levels of connectivity. To address this challenge, we introduce the heterogeneous block covariance model (HBCM) that defines a community structure within the covariance matrix, where edges have signed and continuous weights. Furthermore, it takes into account the heterogeneity of objects when forming connections with other objects within a community. A novel variational expectation-maximization algorithm is proposed to estimate the group membership. The HBCM provides provable consistent estimates of memberships, and its promising performance is observed in numerical simulations with different setups. The model is applied to a single-cell RNA-seq dataset of a mouse embryo and a stock price dataset. Supplementary materials for this article are available online.
\end{abstract}

\noindent%

\section{Introduction}

Cluster analysis is a statistical technique that groups a set of objects into different clusters, ensuring that objects in the same cluster share more similarities with each other than with those from different clusters \citep{macqueen1967some,jain2010data}. This technique serves as a primary task for exploratory analysis and has found applications in numerous fields, including machine learning, data compression, bioinformatics, image analysis, and computer graphics \citep{xu2005survey,kaufman2009finding,gan2020data}.  

Clustering methods are typically designed for grouping data points in a (possibly high-dimensional) Euclidean space into clusters. For example, K-means~\citep{kmeans} searches for cluster centers and assigns data points to the nearest center such that the sum of squared distances from the data points to the corresponding cluster centers is minimized. Spectral clustering \citep{ng2002spectral} employs a similarity matrix of data points based on a certain distance metric and performs cluster analysis using eigenvectors. 

A large class of clustering methods is based on mixture models, which assume that data points are sampled from different sub-populations. For example, in a Gaussian mixture model \citep{Fraley2002}, such sub-populations are modeled by normal distributions with various parameters. The selection of variables has been studied when the distributions are high-dimensional \citep{pan2007penalized,wang2008variable,bouveyron2014model}.

We consider a different clustering problem: clustering at the feature/variable level, which naturally arises, particularly in the study of gene expression data, but has not received enough attention in the statistics literature. In the context of gene expression data, a data point represents a biological sample, and a variable represents a gene. In bioinformatics applications, it is desirable to find clusters of highly correlated genes whose expression levels are regulated simultaneously, thus working synergistically to achieve the same biophysical function. Consequently, correlations, rather than similarity in means or distance between expression level vectors, best describe this kind of relationship.
Weighted gene co-expression network analysis (WGCNA) \citep{langfelder2008wgcna} is one of the most popular methods for this purpose, constructing a network based on the sample correlation matrix and employing network analysis tools for cluster analysis. However, such construction is usually ad-hoc. For example, it remains unclear how to choose the optimal threshold when converting the correlation matrix to a simple graph (i.e., a network with undirected and unweighted edges), and such conversion may result in information loss.

We approach the feature clustering problem from a different perspective. Instead of performing cluster analysis on an artificially constructed network, we directly model and conduct clustering using the raw data matrix, where rows represent data points (e.g., samples) and columns represent features (e.g., genes). Existing clustering algorithms such as model-based methods (e.g., Gaussian mixture model) and distance-based approaches (e.g., K-means) are not ideal choices for feature clustering. These algorithms are either designed for data points that are assumed to be independently generated or are based on certain distance metrics on the data points. However, we aim to cluster correlated objects, and the cluster structure is encoded in their covariances. Note that spectral clustering can, in principle, be used in feature clustering; for example, \cite{luo2007constructing} applied it to cluster gene co-expression matrices. Nevertheless, the statistical justification remains unclear.

We propose a heterogeneous block covariance model (HBCM) in which the rows of the data matrix are assumed to follow a multivariate normal distribution, and the covariance matrix depends on both the clustering structure and the individual characteristics of each column. Specifically, we use a covariance matrix to model the covariance structure at the cluster level, which serves as the backbone of the feature-level covariance matrix. Additionally, we introduce heterogeneous parameters for columns to reflect the individual characteristics of forming a connection with other columns. These heterogeneous parameters allow features in the same cluster to have different covariances, providing a more flexible covariance structure than a compound symmetric covariance or correlation structure. Moreover, it ensures that our algorithm is scale invariant.
 
The block covariance model (BCM) without the heterogeneous parameters has been studied in the literature. \cite{huang2010correlation} proposed an efficient algorithm to sample data from a multivariate distribution with a blockwise covariance matrix but did not consider clustering or parameter estimation under this model. \cite{liuqing} proposed a profile likelihood method based on the BCM to estimate the unknown group membership and other parameters.  \cite{bunea2020model} established minimax lower bounds for exact partition recovery under a variant of BCM (referred to as the $G$-block covariance model in the paper) and proposed clustering algorithms with guaranteed exact recovery of clusters. However, the more complex covariance structure allowed by the HBCM has not been considered in the literature. The gene-level heterogeneous parameters in the HBCM allow for more gene-specific variation besides the common covariance within clusters. Moreover, we study the HBCM from a likelihood-based perspective. To overcome the computational burden, we introduce an equivalent form of the HBCM by adding a second layer of latent variables, resulting in an efficient variational expectation-maximization (EM) algorithm. Additionally, the covariance structure in the HBCM is different from the block-diagonal covariance structure studied in the literature \citep{sun2014adaptive,tan2015cluster,devijver2018block}, where the off-diagonal blocks (or the expected values in a Bayesian setting) are zero.


 

Covariance-based clustering is related to another widely-studied area – community detection in networks, which aims to partition the node set into cohesive groups with many links within and fewer links between \citep{Newman2004Review, Fortunato2010}. The stochastic block model (SBM) is perhaps the best studied model for community detection \citep{abbe2017community, zhao2017survey}. Recently, the model has been generalized to weighted networks \citep{barrat2004architecture, jog2015information, palowitch2017significance, xu2020optimal}, called weighted stochastic block models (WSBM). With variables as nodes and covariances between variables as edges, the sample covariance/correlation matrix can be considered as a weighted network. As aforementioned, the population covariance matrix in the Heterogeneous Block Covariance Model (HBCM) has a block structure, similar to the WSBM, and the introduction of heterogeneous parameters is inspired by the degree-corrected stochastic block model (DCSBM). However, the data structure and the model (HBCM) considered in this paper are substantially different from those in community detection. The weighted network itself is usually treated as the raw data in community detection, and the edges are assumed to be independently generated. In contrast, the entries in the sample covariance matrix cannot be assumed to be independent of each other and, therefore, cannot be modeled under the WSBM framework.

The contribution of this paper is three-fold. Firstly, we propose the HBCM, which incorporates both individual-level and community-level information in the covariances between features. This makes the model applicable to a wide range of data with community structures on features. Secondly, we develop a novel variational EM algorithm to efficiently optimize SBM-like profile-likelihoods. This algorithm addresses the high computational cost associated with searching for the best label assignment over all possibilities. The key innovation is the introduction of a second layer of latent variables, in addition to group membership, enabling the M-steps to have a closed-form solution and reducing the overall computational complexity to polynomial time. Thirdly, our theoretical study shows the identifiability of the group membership and other parameters in the HBCM and the consistency of the estimated group membership with the true group membership under mild technical conditions. By combining these three contributions, the paper presents a comprehensive and efficient approach for feature clustering and provides theoretical guarantees for the estimated group membership.


\section{Methodology}
\subsection{Model description}

Let $\M{X}=[X_{ij}]$ denote an $N\times P$ data matrix. In the context of gene expression data, each row of $\M{X}$ represents a biological sample, and each column represents a gene. Without loss of generality, we assume that the data have been centered. That is, the rows independently follow a $P$-variate normal distribution $N(0, \M{\Sigma})$, where the covariance structure $\M{\Sigma}$ reflects the regulation of gene expression. Our goal is to cluster the $P$ columns into $K$ non-overlapping communities. On a high level, a key difference between the model described below and most classical statistical models in clustering analysis is that we aim to cluster the correlated features rather than samples, and the community structure is encoded in the covariance matrix of $\M{X}$.

We now describe the model in detail. Throughout the paper, we use $i$ as the row index, $j$ as the column index, and $k, l$ as the indices for communities. Group membership is denoted by $\V{c}=(c_1,...,c_P)$ with $c_j\in [K]$, where $[K]$ is a short notation for the set of integers ${1,...,K}$. Group membership is treated as unknown throughout the paper. The covariance structure at the community level is characterized by a $K\times K$ symmetric matrix $\M{\Omega} =[\omega_{kl}]$. 

We assume that each $X_{ij}$ has mean zero, so the community information is contained in the covariance structure. Moreover, the rows of $\M{X}$ are independently and identically distributed given $\V{c}$. The population covariance between any pair $(X_{ij},X_{ij'})$ given $\V{c}$ is of the form
\begin{equation}\label{eq:1.1}
 \Sigma_{jj'}=\text{Cov}(X_{ij}, X_{ij'}|\V{c}) =
\begin{cases}
\lambda_j\lambda_{j'}\omega_{c_j c_{j'}} & j \neq j' \\
\lambda_j^2\omega_{c_j c_{j}} + \sigma_j^2 &  j = j' 
\end{cases},   
\end{equation}
where $\lambda_j\in \mathbb{R}\backslash\{0\}$ and $\sigma_j^2 \in \mathbb{R}^+$. Let $\V{\lambda}=(\lambda_1,...,\lambda_P)$ and $\V{\sigma}^2=(\sigma_1^2,...,\sigma_P^2)$. 
Note that $\M{\Omega}$ needs to be positive semi-definite to make $\M{\Sigma}$ a well-defined covariance matrix for arbitrary $\V{\sigma}^2$. 
This assumption implies that the within-community connections are generally stronger than the between-community connections, which is a common assumption in community detection \citep{Newman2004Review, Fortunato2010}. Refer to Lemma S.1 in the Supplementary Materials for further  details.

We call the model \eqref{eq:1.1} a heterogeneous block covariance model (HBCM). It is of a similar format as the DCSBM for networks with binary edges \citep{Karrer10, Zhaoetal2012}. The matrix $\M{\Omega}$ plays a similar role as the block matrix, where each element represents the edge probability of two nodes given labels. The parameters in $\V{\lambda}$ play similar roles as the degree parameters, reflecting the propensity of each feature to form relationships with other features. There are three key differences between model \eqref{eq:1.1} and the DCSBM. Firstly, our model is concerned with covariance structure, while the DCSBM depicts the mean structure of a network. Secondly, unlike the degree parameters in the DCSBM, the analogous parameters ${\lambda_j}$ in the HBCM can be either positive or negative. This means that covariances in the same community can potentially have opposite signs. This allows negatively correlated genes to be in the same community, which has been discovered in the bioinformatics literature \citep{dhillon2003diametrical,zeng2010maximization,matsumura2011epigenetic}. Thirdly, the model in \eqref{eq:1.1} contains another set of parameters $\V{\sigma}^2$, which do not have an analog in the DCSBM. Note that without $\V{\sigma}^2$, within-community correlations must be 1 or -1. The parameters $\V{\sigma}^2$ and $\V{\lambda}$ make our model more flexible to handle complex gene expression data.

The HBCM includes two special cases that have been studied before. The model is closely related to the BCM \citep{liuqing} when $\lambda_j/\sigma_j$ is equal to a constant for all $j\in[P]$. In this case, the correlation matrix becomes a block matrix. Additionally, the HBCM is reduced to the $G$-block covariance model studied in \cite{bunea2020model} when $\lambda_j=1$ for all $j\in[P]$. However, in general, the HBCM is equivalent to neither of these two models and accommodates a more flexible community structure on covariances.

Note that two different sets of parameters $\{\V{c}, \bome, \blam, \bsig^2\}$ and $\{\tilde{\V{c}}, \tilde{\bome}, \tilde{\blam}, \tilde{\bsig}^2\}$ may define the same covariance $\bSig$. First of all, the group memberships of two systems may be up to a permutation, which is a common issue and not problematic. Moreover, $\blam$ and $\bome$ are not uniquely defined. For example, $\tilde\blam=\blam/2$ and $\tilde\bome=4\bome$ would give an equivalent model. We will study the identifiability problem and the well-definedness of the group membership in Section \ref{sec:identi}. Specifically, Theorem \ref{thm1} therein ensures that the ambiguity of the parameter system comes from only these two sources. Importantly, the number of communities $K$ and the membership up to a permutation are uniquely defined.

With all the model assumptions above, we can write the log-likelihood function of $\M{X}$ (given labels and parameters) as, up to a constant, 
\begin{equation}
    L(\M{X}|\V{c},\bome,\blam, \bsig^2) =  - \frac{N}{2}\log|\M{\Sigma}|-\frac{1}{2}\tr[\M{S}\M{\Sigma}^{-1}], \label{eq:1.2}
\end{equation}
where 
$\M{S}=\frac{1}{N}\M{X}^{\top}\M{X}$ is the sample covariance matrix of the centered $\M{X}$. 


\subsection{Variational expectation-maximization algorithm}\label{sec:alg}

To perform community detection based on \eqref{eq:1.2}, one can, in principle, employ the profile-likelihood technique \citep{Bickel&Chen2009, liuqing}. This involves treating $\V{c}$ as unknown parameters, rewriting \eqref{eq:1.2} as a profile-likelihood function of $\V{c}$ by maximizing over the other parameters given $\V{c}$, and  searching for the optimal $\V{c}$ that maximizes the profile-likelihood. Due to the intractability of exploring all possible $K^P$ solutions, heuristic search techniques such as Tabu search can be applied \citep{Glover&Laguna1997}. Nevertheless, it is important to note that such algorithms still incur significant computational costs. Additionally, maximizing \eqref{eq:1.2} for a fixed $\V{c}$ poses its own challenges, further complicating the execution of the profile-likelihood technique.

To develop a computationally efficient algorithm for maximizing \eqref{eq:1.2}, we will consider the unknown labels $\V{c}$ as latent random variables. In cluster analysis, the treatment of cluster labels can take two common approaches: they can either be treated as fixed unknown parameters or as latent variables, depending on convenience and specific modeling requirements. For example, the EM algorithm for estimating a Gaussian mixture model presents a soft version of the K-means algorithm by incorporating latent variables to capture the uncertainty associated with cluster assignments \citep{hastie2001}.

In this section, we present an efficient variational EM algorithm for the HBCM, which stands as one of the key contributions in this paper. This algorithm involves introducing an additional set of latent variables alongside the group membership $\V{c}$.


Let each $c_j$ independently follow Multinomial$(1,\V{\pi})$ with $\V{\pi}=(\pi_1,...,\pi_K)$. 
The joint log-likelihood function of $\M{X}$ and $\V{c}$, up to a constant, is 
\begin{align*}
    L(\M{X},\V{c}|\Phi) = \sum_{k=1}^K P_k \log \pi_k  - \frac{N}{2}\log|\M{\Sigma}|-\frac{1}{2}\tr[\M{S}\M{\Sigma}^{-1}],
\end{align*}
where $\Phi$ is the set of all parameters, i.e., $\Phi=\{\V{\pi}, \M{\Omega},\V{\lambda},\V{\sigma}^2 \}$, and $P_k$ is the number of members in group $k$, i.e., $P_k = \sum_{j=1}^P 1(c_j=k)$ with $1(\cdot)$ as the indicator function.

It is natural to apply the EM algorithm to estimate the model when $\V{c}$ is latent. As highlighted in \cite{neal1998view}, the EM algorithm can be interpreted as a coordinate ascent algorithm. In other words, it optimizes the objective function by iteratively updating the parameter estimates and the posterior distributions of latent variables. We introduce this perspective in our context. Let $f(\cdot)$ be the marginal likelihood function of $\M{X}$,  $L(\cdot)$ be the log-likelihood counterpart, and $q(\V{c})$ be a probability measure on $\V{c}$. By Jensen's inequality, $L(\M{X}|\Phi) = \log f(\M{X}|\Phi) =\log \sum_{\V{c} \in [K]^P} q(\V{c})\frac{f(\M{X},\V{c}|\Phi)}{q(\V{c})}  \geq J(q(\V{c}),\Phi)$, where 
\begin{equation} \label{eq:1.3}
\begin{split}
    J(q(\V{c}),\Phi) = \sum_{\V{c} \in [K]^P} q(\V{c})L(\M{X},\V{c}|\Phi) - \sum_{\V{c} \in [K]^P} q(\V{c})\log q(\V{c}). 
\end{split}
\end{equation}
 For any given $\Phi$,  $J(q(\V{c}),\Phi)$ is maximized at $q(\V{c})=f(\V{c}|\M{X},\Phi)$, i.e., equality holds in Jensen's inequality, where $f(\V{c}|\M{X},\Phi)$ is the conditional probability mass function for labels. Therefore, the maximizer of $J(q(\V{c}),\Phi)$ over $\Phi$ is the same as that of $L(\M{X}|\Phi)$. In practice, the EM algorithm iteratively maximizes $J(q(\V{c}),\Phi)$ with respect to the parameters $\Phi$ and the probability measure $q(\V{c})$.


 There are two challenges when directly implementing the classical EM algorithm. Firstly, the EM algorithm for optimizing $L(\M{X},\V{c}|\Phi)$ faces a unique issue: the complete likelihood $L(\M{X}, \V{c}|\Phi)$ does not take the form of $\sum_{m=1}^M T_m(\M{X},\V{c})\eta_m(\Phi)$, i.e., the summation of products between the sufficient statistics of $(\M{X},\V{c})$ and functions of the parameters $\Phi$, where $M$ is the number of terms in the summation. If $L(\M{X}, \V{c}|\Phi)$ can be written as such a form, then up to a constant,
\begin{align*}
\max_{\Phi} J(q(\V{c}),\Phi) = \max_{\Phi} \sum_{m=1}^M \left( \sum_{\V{c} \in [K]^P} q(\V{c})  T_m(\M{X},\V{c}) \right)\eta_m(\Phi).
\end{align*}
Suppose one needs to use a certain numerical algorithm (e.g., Newton–Raphson) to update $\Phi$. In such a scenario, $\sum_{\V{c} \in [K]^P} q(\V{c}) T_m(\M{X},\V{c})$ becomes a constant and only needs to be computed once during an M-step. Without this particular form, the separation between the summation and the maximization cannot be achieved, resulting in a significant increase in computational complexity within the M-step. Secondly, the summation across $K^P$ terms becomes intractable, even for a moderately-sized $P$. This issue is a common challenge encountered when applying the EM algorithm in network community detection contexts \citep{zhao2017survey, amini2013pseudo}.

In the remainder of this section, we delve into the details of how to address the two challenges and provide a comprehensive account of the algorithm. To resolve the first issue, we introduce a second layer of latent variables $\V{\alpha}=[\alpha_{ik}]_{N\times K}$ in conjunction with the latent group membership $\V{c}$. We propose the following data-generating process:
\begin{enumerate}
	\item  For $j=1,...,P$, generate $c_j$ independently and identically from a multinomial distribution with parameter $\V{\pi}=(\pi_1,...,\pi_K)$;
	\item  For $i=1,...,N$, generate $\V{\alpha}_i=(\alpha_{i1},...,\alpha_{iK})$ independently and identically from the $K$-variate normal distribution $N(\V{0},\M{\Omega})$; 
	\item For each $i,j$, conditional on $\V{\alpha}_i$ and $c_j$, let $X_{ij}=\lambda_j \alpha_{ic_j}+ \sigma_j \epsilon_{ij}$, where $\epsilon_{ij}$ is independently drawn from the standard normal distribution $N(0,1)$.

\end{enumerate}
The inclusion of the second latent layer ($\V{\alpha}_i$ in step 2) draws inspiration from random-effects models, where a random-effect term is incorporated to accommodate additional sources of variance beyond the sampling error. It is straightforward to verify that the covariance structure derived from the new  data-generating process is equivalent to the one posited in~\eqref{eq:1.1}. With the introduction of the second latent layer, the complete log-likelihood $L(\M{X},\V{c},\V{\alpha}|\Phi)$ belongs to the exponential family, as shown below:
\begin{align*}
L(\M{X},\V{c},\V{\alpha}|\Phi) 
= \,\, & \sum_{j=1}^P\sum_{k=1}^K 1(c_j=k)\log{\pi_k} -\frac{N}{2}\log|\M{\Omega}|-\frac{1}{2}\sum_{i=1}^N\V{\alpha}_i^\top\M{\Omega}^{-1}\V{\alpha}_i \\
    &  +  \sum_{j=1}^P\sum_{i=1}^N\sum_{k=1}^K 1(c_j=k)\bigg( -\frac{1}{2}\log\sigma_j^2-\frac{(X_{ij}-\lambda_j\alpha_{ik})^2}{2\sigma_j^2}\bigg).
\end{align*}
By Jensen's inequality, the objective function of the EM algorithm becomes
\begin{align*}
J(q(\V{c},\V{\alpha}), \Phi)&=\sum_{\V{c} \in [K]^P}\int_{\V{\alpha}\in \mathbb{R}^{N\times K}} q(\V{c}, \V{\alpha})L(\M{X},\V{c},\V{\alpha}|\Phi)d\V{\alpha} -\sum_{\V{c} \in [K]^P}\int_{\V{\alpha}\in \mathbb{R}^{N\times K}} q(\V{c},\V{\alpha})\log q(\V{c},\V{\alpha}) d\V{\alpha},
\end{align*}
where $q(\V{c},\V{\alpha})$ is the joint posterior distribution of $\V{c}$ and $\V{\alpha}$.

Subsequently, we address the second challenge, namely the intractable summation over $K^P$ terms and the integration over a high-dimensional space involved in computing $q(\V{c},\V{\alpha})$. The variational EM \citep{blei2017variational} algorithm serves as an alternative to the classical EM algorithm when the computational feasibility of the model under consideration is compromised. The variational EM employs an approximate inference technique to estimate the latent variables. To elaborate, the variational EM algorithm introduces constraints on the posterior distribution to render it a particular tractable form. The inference step then seeks the constrained posterior distribution that best approximates the exact posterior distribution. In our scenario, we introduce the variational assumption that the posterior distributions of the two sets of latent variables are independent, i.e., $q(\V{c},\V{\alpha})=q_1(\V{c})q_2(\V{\alpha})$. This assumption is an approximation because the independence between the posterior distributions of $\V{c}$ and $\V{\alpha}$ is not guaranteed. The inspiration for imposing this constraint arises from the latent block model (LBM) \citep{govaert2008block, wang2021efficient}, a bi-clustering model that aims to simultaneously cluster rows and columns. The LBM employs the variational approximation $q(\V{c},\V{z})=q_1(\V{c})q_2(\V{z})$ for column labels $\V{c}$ and row labels $\V{z}$. Although our work does not encompass bi-clustering, the joint distribution of the latent variables $\V{c}$ and $\V{\alpha}$ shares a similar dependency structure as $\V{c}$ and $\V{z}$ in the LBM. The rationale behind the variational approximation is that if the identified communities align with the true structure, the posterior distribution $q_1(\V{c})$ will become a Dirac measure concentrated on the true group membership. As a Dirac measure, it becomes independent of $q_2(\V{\alpha})$.

Given the variational assumption, we write the objective function of the variational EM algorithm as
\begin{align} 
 J(q_1(\V{c}),q_2(\V{\alpha}), \Phi) = &\sum_{\V{c} \in [K]^P}\int_{\V{\alpha}\in \mathbb{R}^{N\times K}} q_1(\V{c})q_2( \V{\alpha})L(\M{X},\V{c},\V{\alpha}|\Phi)d\V{\alpha} \nonumber \\
  & - \sum_{\V{c} \in [K]^P}\int_{\V{\alpha}\in \mathbb{R}^{N\times K}} q_1(\V{c})q_2( \V{\alpha})\log (q_1(\V{c})q_2( \V{\alpha}))d\V{\alpha}. \label{eq:1.4}
\end{align}
The algorithm then endeavors to maximize \eqref{eq:1.4} by iteratively updating $\Phi$, $q_1(\V{c})$, and $q_2(\V{\alpha})$ in alternating steps. The specifics of this process are outlined in Algorithm \ref{algo}, while its justification is provided in the supplementary materials.
 In Algorithm \ref{algo}, notation $\mathbb{E}_{q_1}[\cdot]$ and $\mathbb{E}_{q_2}[\cdot]$ denote the expectations with respect to $q_1(\V{c})$ and $q_2(\V{\alpha})$, respectively. That is, for all $h$,
\begin{align*}
\mathbb{E}_{q_1}[h] & = \sum_{\V{c} \in [K]^P} h q_1(\V{c}), \,\, \mathbb{E}_{q_2}[h]  = \int_{\V{\alpha}\in \mathbb{R}^{N\times K}} h q_2( \V{\alpha})d\V{\alpha}.
\end{align*}

 There are two noteworthy details to highlight. Firstly, as demonstrated in Algorithm~\ref{algo}, both the parameter estimates in the M-step and the updates to the posterior distributions in the E-step possess closed-form expressions, leading to a further reduction in computational expenses. Secondly, the posterior distributions $q_1(\V{c})$ and $q_2(\V{\alpha})$ can be factorized without the need for additional constraints.

 Specifically, given $q_2(\V{\alpha})$ and $\Phi$, $q_1(\V{c})$ can be factorized as the product of $P$ distributions of the columns, i.e. $q_1(\V{c})=\prod_{j=1}^Pq_{1j}(c_j)$. Similarly, given $q_1(\V{c})$ and $\Phi$, $q_{2}(\V{\alpha})$  can be factorized as a product of $N$ distributions of the rows, i.e. $q_2(\V{\alpha})=\prod_{i=1}^N q_{2i}(\V{\alpha}_i)$. Furthermore, each $q_2(\V{\alpha}_i)$ is in fact the density function of a $K$-variate normal distribution $N(\hat{\V{\mu}}_i, \hat{\M{V}})$, where
\begin{align}
    \hat{\V{\mu}}_i & = \left(\hat{\M{\Omega}}^{-1}+\sum_{j=1}^P\M{D}_j\right)^{-1}\left(\sum_{j=1}^P\M{D}_j\V{b}_{ij}\right), \,\, \hat{\M{V}} = \left(\hat{\M{\Omega}}^{-1}+\sum_{j=1}^P\M{D}_j\right)^{-1} \label{post_parameter},
\end{align} 
\begin{align*}
\M{D}_{j}&= \frac{\hat{\lambda}_j^2}{\hat{\sigma}_j^2}\begin{pmatrix}
\mathbb{E}_{q_1}(1(c_j=1)) & 0 & \cdots & 0 \\
0 & \mathbb{E}_{q_1}(1(c_j=2)) & \cdots & 0 \\
\vdots & \vdots & \ddots & \vdots\\
0 & 0 & \cdots & \mathbb{E}_{q_1}(1(c_j=K))
\end{pmatrix}_{K \times K}, \,\,
    \V{b}_{ij} &=\frac{1}{\hat{\lambda_j}}\begin{pmatrix}
    X_{ij}\\
    X_{ij}\\
    \vdots \\
    X_{ij}
    \end{pmatrix}_{K\times 1},
\end{align*}
and ``hat'' represents the current estimate from M-steps. 
The factorization of $q_1(\V{c})$ reduces the number of terms in the sum from $K^P$ to $KP$, and the factorization of $q_2(\V{\alpha})$ changes the integration from over a $NK$-dimensional space to $N$ separate integrations in $K$-dimensional space, which dramatically reduces the computational complexity and makes the computation tractable. The readers are referred to the supplemental materials for detailed justifications.

\begin{algorithm}
\caption{Variational EM Algorithm for HBCM}
\label{algo}
\textbf{Input:} $\M{X}$, $K$, initial values $\hat{\V{c}}^{(0)},\hat{\V{\pi}}^{(0)}, \hat{\M{\Omega}}^{(0)}, \hat{\V{\lambda}}^{(0)}, \hat{\V{\sigma}}^{2(0)}$;

\Repeat{convergence}{
\textbf{E-step:}        
\begin{align*}
& \hat{q}_1(\V{c}) =\prod_{j=1}^P\frac{f_j(c_j)}{\sum_{k=1}^K f_j(k)}, \\
& \textnormal{where }  f_j(k) =\exp \left\{\log \hat{\pi}_{k}-\frac{N}{2}\log{\hat{\sigma}_j^2}- \sum_{i=1}^N\frac{\big(X_{ij}^2-2\hat{\lambda}_jX_{ij}\mathbb{E}_{q_2}[\alpha_{ik}]+\hat{\lambda}_j^2\mathbb{E}_{q_2}[\alpha_{ik}^2]\big)}{2\hat{\sigma}_j^2} \right\}. \\
& \hat{q}_2(\V{\alpha})=\prod_{i=1}^N{\hat{q}_2(\V{\alpha}_i)}, \\
& \textnormal{where $\hat{q}_2(\V{\alpha}_i)$ is the density function of $N(\hat{\V{\mu}}_i,\hat{\M{V}})$ with the parameters given in \eqref{post_parameter}.
}
\end{align*}

\textbf{M-step:}
\begin{align*} 
\hat{\M{\Omega}} & = \frac{1}{N}\sum_{i=1}^N\mathbb{E}_{q_2}(\V{\alpha}_i\V{\alpha}_i^\top),\\
             \hat{\pi}_k & = \frac{\sum_{j=1}^P \mathbb{E}_{q_1}[1(c_j=k)]}{\sum_{j=1}^P\sum_{k'=1}^K \mathbb{E}_{q_1}[1(c_j=k')]}, \quad k=1,...,K,
            \\
             \hat{\lambda}_j & = \frac{\sum_{i=1}^N\sum_{k=1}^K \mathbb{E}_{q_1}[1(c_j=k)]X_{ij}\mathbb{E}_{q_2}[\alpha_{ik}]}{\sum_{i=1}^N\sum_{k=1}^K \mathbb{E}_{q_1}[1(c_j=k)]\mathbb{E}_{q_2}[\alpha_{ik}^2]}, \quad  j= 1,...,P,\\
             \hat{\sigma}_j^2 & = \frac{\sum_{i=1}^N\sum_{k=1}^K \mathbb{E}_{q_1}[1(c_j=k)]
    (X_{ij}^2+\hat{\lambda}_j^2\mathbb{E}_{q_2}[\alpha_{ik}^2]-2\hat{\lambda}_jX_{ij}\mathbb{E}_{q_2}[\alpha_{ik}])}{N},\quad  j= 1,...,P.
\end{align*}
}
$$\hat{c}_j=\underset{k={1,...,K}}{\text{argmax}}\, f_j(k).$$

\textbf{Output:} estimated group membership $\hat{\V{c}}$, parameter estimation $\hat{\Phi}$.
\end{algorithm}

The variational EM algorithm requires initial parameter estimates to initiate the iterative updating process. In practice, our approach involves initially applying spectral clustering \citep{ng2002spectral} to the absolute sample correlation matrix, yielding an initial estimate of the group membership. Here, we view the sample correlation matrix as a signed, weighted graph where edges represent correlation values. Taking the absolute value of correlations avoids negative degrees in the graph.
To add randomness to the initial values, we create a $P \times K$ probability matrix to serve as the initial values of matrix $q_1(\mathbf{c})$. Each row of $q_1(\mathbf{c})$ is a vector of length $K$ representing the probabilities of column $j$ belonging to different clusters, with the sum equal to 1. We randomly assign a probability between 0 and 0.5 to entry $k$ corresponding to the initial assignment from spectral clustering and then randomly distribute the remaining probabilities to other entries.
Subsequently, we set the initial estimate $\hat{\pi}_k^{(0)}$ by averaging the columns of $q_1(\mathbf{c})$, and establish the estimate $\hat{\omega}_{kl}^{(0)}$ as the average sample covariance for all pairs of columns, where one column belongs to group $k$ and the other to group $l$.

For the initial estimates of $\V{\lambda}$ and $\V{\sigma}^2$, we employ a simplified version of Algorithm~\ref{algo}. Specifically, we gather the columns belonging to cluster $k$ and apply the algorithm to update $\hat{q}_{2i}(\balp_i)$ and the corresponding parameters $\hat{\lambda}^{(0)}_j$ (for ${j: c_j=k}$) and $\hat{\sigma}_j^{2(0)}$ (for ${j: c_j=k}$). Given that all the subset columns pertain to a single group, there is no need to update $\hat{q}_{1j}(c_j)$, $\hat{\bome}$, and $\hat{\pi}_k$ within the simplified algorithm. This process is repeated for each group $k\in [K]$. The updated $\hat{\V{\lambda}}^{(0)}$, $\hat{\V{\sigma}}^{2(0)}$, along with $\hat{\V{\pi}}^{(0)}$ and $\hat{\M{\Omega}}^{(0)}$, collectively serve as the initial estimates in Algorithm~\ref{algo}.

\section{Theory}
\subsection{Model Identifiability}\label{sec:identi}
Before studying the asymptotic properties, we clarify the identifiability of our parametrization system. Assume that there are $K$ communities among $P$ features. Let $\M{L}=[L_{jk}]$ be a $P\times K$ label matrix such that $L_{jk}=1$ if $c_j=k$ and $L_{jk}=0$ otherwise. That is, there is a unique ``1'' in each row of $\M{L}$ indicating the community label of the row index. Then covariance structure is written as 
\begin{equation}\label{e6}
\bSig=\diag(\blam)\M{L}\bome\M{L}^{\top}\diag(\blam)+\diag(\bsig^2),    
\end{equation}
where $\blam$, $\bsig^2$, $\bome$ were defined earlier. The covariance matrix is determined by the parameter system $\{\blam,\bome, \M{L}, \bsig^2\}$, where the membership is encoded in the matrix $\M{L}$.

It is meaningful to ask whether we can infer the community label from $\bSig$ given \eqref{e6}, as two different systems $\{\blam,\bome, \M{L}, \bsig^2\}$ and $\{\tilde\blam,\tilde\bome,\tilde{\M{L}},\tilde{\bsig}^2\}$ may lead to a same covariance $\bSig$. Indeed, the community labels of two system may be up to a permutation. It is a well-known issue that can be easily taken care in practice. In addition, the product structure in \eqref{e6} implies that $\blam$ and $\bome$ are defined up to a multiplier. For instance, $\blam/2$ and $4\bome$ would give an equivalent model. A serious question is, given $\bSig$, whether the number of communities $K$ and the membership are uniquely defined. Theorem \ref{thm1} below characterizes all equivalent parameter systems and clarifies the issue, confirming the membership in the HBCM is well-defined. 


Condition 1. For a parameter system $\{\blam,\bome, \M{L}, \bsig^2\}$, we assume that $\blam$ is a $P$-vector with non-zero entries, $\bome$ is a $K\times K$ positive definite matrix, $\M{L}$ is a $P\times K$ matrix with all entries equal to 0 or 1, and $\M{L}\M{1}_K=\M{1}_P$, $\M{L}^{\top}\M{1}_{P}\geq 3\cdot\M{1}_K$ (entry-wise comparison), $\bsig^2$ is a $P$-vector with positive entries. 

The only additional assumption in condition 1 is $\M{L}^{\top}\M{1}_{P}\geq 3\cdot\M{1}_K$, which requires that each community contains at least 3 features.

\begin{theorem}\label{thm1}
Two systems $\{\blam,\bome, \M{L}, \bsig^2\}$ and $\{\tilde\blam,\tilde\bome,\tilde{\M{L}},\tilde{\bsig}^2\}$ under Condition 1 give the same covariance structure $\bSig$ if and only if $\tilde\bome\ =\M{P}^{\top}\diag({\V{d}})\bome\diag({\V{d}})\M{P}$, $\tilde{\M{L}} = \M{L}\M{P}$, $\tilde{\blam} =\blam\circ (\M{L}\V{d}^{-1})$, $\bsig^2=\tilde\bsig^2$ where $\M{P}$ is a $K\times K$ permutation matrix, $\V{d}$ is a $K$ dimensional vector with nonzero entries, $\V{d}^{-1}$ is the vector of entry-wise reciprocals, $\circ$ is entry-wise (Hadamard) product.
\end{theorem}

Theorem \ref{thm1} concludes that the ambiguity of the parameter system comes from only two sources. One is the label permutation, and the other is the scale distribution between $\blam$ and $\bome$. When $\M{P}$ is identity, Theorem \ref{thm1} asserts $\tilde{\M{L}} = \M{L}$, which indicates that the community labels are well-defined and completely determined by the covariance structure.

From now on, we will assume all parameter systems employed in this paper are identical in labels so we ignore the permutation matrix $\M{P}$. Theorem \ref{thm1} indicates $\blam$, $\bome$ are not identifiable in a sense that a reparameterization $\tilde\lambda_j=d_{c_j}^{-1}\lambda_j$ and $\tilde\Omega_{kl}=d_kd_l\Omega_{kl}$ gives the same covariance model. This will not affect the statement of our main result on community detection consistency but makes it troublesome to illustrate intermediate steps. To resolve this issue, we pick the following canonical parameters:
\begin{align*}
    \lambda_j^* &=\frac{1}{P} \left(\sum_{j'=1}^P\lambda_{j'}\omega_{c_{j'}c_j}\right)\lambda_j, \,\, \omega_{c_jc_{j'}}^*  = \frac{\lambda_j\lambda_{j'}\omega_{c_jc_{j'}}}{\lambda_j^*\lambda_{j'}^*} = \frac{P^2\omega_{c_jc_{j'}}}{(\sum_{t=1}^P\lambda_t\omega_{c_jc_t})(\sum_{s=1}^P\lambda_s\omega_{c_{j'}c_s})},
\end{align*}
which are well-defined as long as $\sum_{j'=1}^P\lambda_{j'}\omega_{c_{j'}c_j}\ne 0$ for all $j\in[P]$. We can rewrite the definition by matrix notation
\begin{align*}    \blam^*&=\frac1P\diag(\blam)\M{L}\bome\M{L}^{\top}\diag(\blam)\M{1}_P=\blam\circ(\M{L}\bome\M{L}^{\top}\diag(\blam)\M{1}_P/P) =\blam\circ(\M{L}\V{d}^{-1}),\\
    \bome^*&=\diag({\V{d}})\bome\diag({\V{d}}),
\end{align*}
where $\V{d}=(\bome\M{L}^{\top}\diag(\blam)\M{1}_P/P)^{-1}$.
\begin{proposition}
Assume that $\sum_{j'=1}^P\lambda_{j'}\omega_{c_{j'}c_j}\ne 0$ for all $j\in[P]$. If two systems $\{\blam,\bome, \M{L}, \bsig^2\}$ and $\{\tilde\blam,\tilde\bome,\tilde{\M{L}},\tilde{\bsig}^2\}$ define the same covariance structure $\bSig$, then they lead to the same canonical parameters $\blam^*=\tilde\blam^*$, $\bome=\tilde\bome^*$. 
\end{proposition}
The canonical parameter is a representative in each equivalent class of parameters that define the same covariance structure. If the data are generated from some ``true'' parameters $\blam$ and $\bome$, there is no way to estimate them because of the identifiability issue. Nevertheless, we show later that it is possible estimate the canonical parameter defined above, and consistently estimate the community membership. 

It is a common practice to standardize the features before statistical analysis. Therefore, scale invariance is a desired property a statistical procedure. The following proposition shows that the membership matrix $\M{L}$ remains the same after a scale change. 
\begin{proposition}
(Scale invariance) Let $\V{b}$ be a $p$-dimensional vector with non-zero entries. If $\V{X}_1$ is generated from the current model $\bSig$ with a parameter system $\{\blam,\bome, \M{L}, \bsig^2\}$, then the rescaled $\V{X}_1$ has Cov$(\V{b}\circ\V{X}_1)=\diag(\V{b})\bSig\diag(\V{b})$ that can be generated from the parameter system $\{\V{b}\circ\blam,\bome, \M{L}, \bsig^2\}$. In particular, the membership matrix $\M{L}$ remains unchanged.     
\end{proposition}

\subsection{Asymptotic Properties}
In this section, we establish the consistency of the estimated group membership, which is the consistency of the optimal distribution $q_1(\V{c})$ of $ J(q_1(\V{c}),q_2(\V{\alpha}), \Phi)$ with the Dirac measure of the true labels. We first define a sample version and a population version of the objective function $ J(q_1(\V{c}),q_2(\V{\alpha}), \Phi)$. 
We begin with the sample version. 
In Section \ref{sec:alg}, we demonstrated that $q_1(\V{c})$  can be factorized  given  parameters and $q_2(\V{\alpha})$, and similarly $q_2(\V{\alpha})$ is product of $n$ multivariate normal density functions given parameters and $q_1(\V{c})$.  We restrict our analysis within the same class. That is, $q_1(\V{c})$ can be represented by a matrix $\bqc = [q_{jk}^c]_{P\times K}$, and $q_{2}(\V{\alpha})=\prod_{i=1}^N q_{2i}(\V{\alpha}_i)$. Furthermore, $q_{2i}(\V{\alpha}_i)$ is the density function of $N(\V{\mu}_i,\M{V}_i)$, where $\V{\mu}_i=(\mu_{i1},...,\mu_{iK})$ and $\textnormal{diag}(\M{V}_i)=(v_{i1}^2,...,v_{iK}^2)$. Finally, let $\V{\mu}=[\mu_{ik}]_{N\times K}$. Define
\begin{align*}
   & \hat{J}(q_1(\V{c}), q_2(\V{\alpha}),\V{\pi},\M{\Omega})  = \sum_{j=1}^P\left(\sum_{k=1}^Kq_{jk}^c\log
   \pi_k\right) + \sum_{i=1}^N\left(-\frac{1}{2}\log|\bome| - \frac{1}{2}\tr{\left(\mathbb{E}_{q_2}(\alpha_i\alpha_i^\top)\bome^{-1}\right)}\right) \\
  & \quad \quad  + \sum_{i=1}^N\sum_{j=1}^P\sum_{k=1}^K\left(-\frac{1}{2}\hat{\lambda}_j^2q_{jk}^c (\mu_{ik}^2+v_{ik}^2 )\right) + \sum_{i=1}^N\sum_{j=1}^P\sum_{k=1}^Kq_{jk}^c\hat{\lambda}_jX_{ij}\mu_{ik}\\
  & \quad \quad - \sum_{j=1}^P\sum_{k=1}^Kq_{jk}^c\log q_{jk}^c - \sum_{i=1}^N\int_{\V{\alpha}_i\in \mathbb{R}^{ K}} q_{2i}( \V{\alpha}_i) \log (q_{2i}( \V{\alpha}_i)) d \V{\alpha}_i,
\end{align*}
where $\hat{\lambda}_j = \sum_{j'=1}^P(\sum_{i=1}^NX_{ij}X_{ij'})/(PN)$.

There are two differences between $J(q_1(\V{c}), q_2(\V{\alpha}),\Phi)$ defined in \eqref{eq:1.4} and $\hat{J}(q_1(\V{c}), q_2(\V{\alpha}),\V{\pi},\M{\Omega})$ for the convenience of the theoretical analysis. The first difference is we use the moment estimator $\hat{\lambda}_j$ as aforementioned for the rest of the analysis. The rest of the parameters including $q_1(\V{c})$ and $q_2(\V{\alpha})$ will be maximized over proper compact sets (see Proposition \ref{thm:separation} for details), similar to the standard treatment in asymptotic theory. We give a proposition on the consistency of $\hat{\lambda}_j$ to the canonical form $\lambda_j^* = \sum_{t=1}^P (\lambda_j\lambda_{j'}\omega_{c_jc_{j'}})/P$. 
\begin{proposition}
As $N,P\rightarrow \infty$ with $\log(P)/N = o(1)$ and $\max_j \sigma_j^{*2}=o(P)$,
\begin{align*}
   & \mathbb{P}(|\hat{\lambda}_j - \lambda_j^*|\geq \epsilon | \V{c}^*) \rightarrow 0.
\end{align*}
\end{proposition}
 The second difference between $J(q_1(\V{c}), q_2(\V{\alpha}),\Phi)$ and $\hat{J}(q_1(\V{c}), q_2(\V{\alpha}),\V{\pi},\M{\Omega})$ is that we omit $\{\sigma_j^2\}$ in $\hat{J}(q_1(\V{c}), q_2(\V{\alpha}),\Phi)$ for the convenience of the theoretical analysis. It does not mean that we assume the true standard deviation $\sigma^{*2}_j \equiv 1$ for all $j$ in the model. Instead, the analysis below implies inconsistent estimates of $\{\sigma^{*2}_j\}$ will not affect the label consistency as long as $\{\sigma^{*2}_j\}$ are properly bounded (see Theorem \ref{thm:consistency} for details). 

A key ingredient in the proof is to demonstrate that except for 
\begin{align*}
       & \hat{J}_{\textnormal{core}}(\bqc, \V{\mu}) =\sum_{i=1}^N\sum_{j=1}^P\sum_{k=1}^K\left(-\frac{1}{2}\hat{\lambda}_j^2q_{jk}^c \mu_{ik}^2 \right) + \sum_{i=1}^N\sum_{j=1}^P\sum_{k=1}^Kq_{jk}^c\hat{\lambda}_jX_{ij}\mu_{ik},
\end{align*}
the other terms in $\hat{J}(q_1(\V{c}), q_2(\V{\alpha}),\V{\pi},\M{\Omega})$ are $o(NP)$ under certain regularity conditions, and their contribution in clustering is therefore negligible. 


Let $\V{c}^*$ and $\V{\alpha}^*$ be the true values of $\V{c}$ and $\V{\alpha}$.  
We now define the population version of as
\begin{align*}
    &\bar{J}(q_1(\V{c}), q_2(\V{\alpha}),\V{\pi},\M{\Omega}) =  \sum_{j=1}^P\left(\sum_{k=1}^Kq_{jk}^c\log
   \pi_k\right) + \sum_{i=1}^N\left(-\frac{1}{2}\log|\bome| - \frac{1}{2}\tr{\left(\mathbb{E}_{q_2}(\alpha_i\alpha_i^\top)\bome^{-1}\right)}\right) \\
  & \quad  \quad  +  \sum_{i=1}^N\sum_{j=1}^P\sum_{k=1}^K\left(-\frac{1}{2}\lambda_j^{*2}q_{jk}^c (\mu_{ik}^2+v_{ik}^2 )\right)  + \sum_{i=1}^N\sum_{j=1}^P\sum_{k=1}^K q_{jk}^c\lambda_j^{*}\mathbb{E}(X_{ij}|\V{\alpha}^*,\V{c}^*)\mu_{ik}\\
  & \quad  \quad - \sum_{j=1}^P\sum_{k=1}^Kq_{jk}^c\log q_{jk}^c - \sum_{i=1}^N\int_{\V{\alpha}_i\in \mathbb{R}^{ K}} q_{2i}( \V{\alpha}_i) \log (q_{2i}( \V{\alpha}_i)) d \V{\alpha}_i,
\end{align*}
and define
\begin{align*}
        \bar{J}_{\textnormal{core}}(\bqc, \V{\mu} ) & =
  \sum_{i=1}^N\sum_{j=1}^P\sum_{k=1}^K\left(-\frac{1}{2}\lambda_j^{*2}q_{jk}^c \mu_{ik}^2\right)  + \sum_{i=1}^N\sum_{j=1}^P\sum_{k=1}^K q_{jk}^c\lambda_j^{*}\mathbb{E}(X_{ij}|\V{\alpha}^*,\V{c}^*)\mu_{ik}.
\end{align*}
Denote $\V{I}^{c} = [I_{jk}^{c}]_{P\times K}$ where $I_{jk}^{c} = 1(c^*_j = k)$. 

Let $S_K$ be the set of permutations on $\{1,...,K\}$.  $(\tilde{\V{I}}^c,\tilde{\V{\alpha}})$ is called equivalent to $(\V{I}^{c},\V{\alpha}^*)$ if for some $s \in S_K$,
\begin{align*}
\tilde{I}^c_{jk} & =I_{j,s(k)}^{c}, \,\,  {j=1,...,P,k=1,...,K}, \\
\tilde{\alpha}_{ik} & =\alpha^*_{i,s(k)}, \,\,  {i=1,...,N,k=1,...,K}. 
\end{align*}
Let $\mathcal{E}_{\V{I}^{c},\V{\alpha}^*}$ be the equivalent class of $(\V{I}^{c},\V{\alpha}^*)$ and let $\mathcal{E}_{\V{I}^c}$ be the equivalent class of $\V{I}^c$. The next proposition states that the population version  $\bar{J}_{\textnormal{core}}(\bqc, \V{\mu})$ has a unique optimizer achieved at $\mathcal{E}_{\V{I}^{c},\V{\alpha}^*}$:
\begin{proposition}
 $\bar{J}_{\textnormal{core}}(\bqc, \V{\mu})$ is maximized by $\omega$ if and only if $\omega \in \mathcal{E}_{\V{I}^{c},\V{\alpha}^*}$.
\end{proposition}
Next we give a proposition stating $\bar{J}(q_1(\V{c}), q_2(\V{\alpha}),\V{\pi},\M{\Omega})$ is asymptotically maximized by $\bar{J}_{\textnormal{core}}(\bic, \V{\alpha}^*)$ up to terms of  $o(NP)$, and the maximizer is well-separated. We give more definitions before proceeding. We first define a soft confusion matrix, which generalizes the classical confusion matrix for probabilistic label assignments. 
\begin{definition}[Soft confusion matrix]
For any label assignment matrices $\bqc$ and $\tilde{\V{q}}^c$, let $\M{R}(\bqc,\tilde{\V{q}}^c)$ be a $K\times K$ matrix, where
\begin{align*}
	\M{R}_{ kk'}(\bqc,\tilde{\V{q}}^c) & =\frac{1}{P} \sum_{j=1}^P q^{c}_{jk} \tilde{q}^{c}_{jk'}.
\end{align*}  
\end{definition}
Similarly to a standard confusion matrix, the trace of $\M{R}(\bqc,\tilde{\V{q}}^c)$ measures the similarity between the two probabilistic label assignments.

Let $\mathcal{C}_{\V{\pi}}= \{ \V{\pi}: \V{\pi} \in \mathbb{R}^{K},\, \pi_k \in [\pi_{\textnormal{min}},\pi_{\textnormal{max}}],\, \sum_{k=1}^K\pi_k=1   \}$, where $0<\pi_{\textnormal{min}}<\pi_{\textnormal{max}}<\infty$;  $\mathcal{C}_{\M{\Omega}}=\{ \M{\Omega}: \M{\Omega} \in \mathbb{R}^{P \times P}, \M{\Omega} \succ 0, \lambda_{\textnormal{min}}\leq \lambda(\M{\Omega})\leq \lambda_{\textnormal{max}} \}$, where $\lambda(\M{\Omega})$ is an eigenvalue of $\M{\Omega}$ and $ 0<\lambda_{\textnormal{min}}<\lambda_{\textnormal{max}}<\infty$; $\mathcal{C}_{\V{c}} = \{\bqc: \bqc \in \mathbb{R}^{P\times K}, 0\leq q_{jk}^c \leq 1, \sum_{k=1}^Kq_{jk}^c = 1\}$; $\mathcal{C}_{\V{\alpha}} = \{q_2(\V{\alpha}):  |\mu_{ik}|\leq B_1, v_{ik}^2\leq B_2\}$, where $B_1>0,B_2>0$.
\begin{proposition} \label{thm:separation}
Assume  $\min_j|\lambda_j^* |\geq \gamma_1>0$, $B_3 = \min_{1\leq k < l \leq K}(\omega_{kk}^*+\omega_{ll}^*-2\omega_{kl}^*)>0$, and $|\alpha_{ik}^*| \leq B_1$. Furthermore, assume $\left (\sum_{j=1}^P 1(c_j^*=1)/P,..., \sum_{j=1}^P 1(c_j^*=K)/P \right )\in \mathcal{C}_{\V{\pi}}$. Then for all $\V{\pi} \in \mathcal{C}_{\V{\pi}}$, $\M{\Omega} \in \mathcal{C}_{\M{\Omega}}$, $q_1(\V{c}) \in \mathcal{C}_{\V{c}}$, and $q_2(\V{\alpha}) \in \mathcal{C}_{\V{\alpha}}$,
\begin{align*} 
& \mathbb{P}\left( \left . \bar{J}_{\textnormal{core}}(\bic, \V{\alpha}^*)+c_1 P+c_2 N\log P \right. \right.\\ 
& \quad \quad \left . \left .-  \bar{J}(q_1(\V{c}), q_2(\V{\alpha}),\V{\pi},\M{\Omega}) \geq c_3 NP \min_{\tilde{\V{I}}^c \in \mathcal{E}_{\V{I}^c} } \left (1-\textnormal{Tr} (\M{R}(\bqc,\tilde{\V{I}}^c))\right ) \right | \V{c}^* \right) \rightarrow 1,
\end{align*}
where $c_1$, $c_2$, and $c_3$ are constants.
\end{proposition}

The next result is about the uniform convergence of $\hat{J}_{\textnormal{core}}(\bqc, \V{\mu})$ to its population counterpart $\bar{J}_{\textnormal{core}}(\bqc, \V{\mu})$.

\begin{proposition}
\label{thm:uniform}
 Assume $\lambda_j^*$ satisfying $\min_j|\lambda_j^* |\geq \gamma_1$ and $\max_j|\lambda_j^*| \leq \gamma_2$. Furthermore, assume $\max_j \sigma_j^{*2}=O(1)$ and $\max_{kl} |\omega_{kl}|=O(1)$. Then for all $ \epsilon > 0$, as $N, P \rightarrow \infty$ and $\log(P)/N = o(1)$,
\begin{align*}
    \mathbb{P}\left( \max_{q_1(\V{c})\in \mathcal{C}_{\V{c}}, q_2(\V{\alpha}) \in \mathcal{C}_{\V{\alpha}}}  \left|\hat{J}_{\textnormal{core}}(\bqc, \V{\mu}) - \bar{J}_{\textnormal{core}}(\bqc, \V{\mu}) \right|\geq NP\epsilon  \middle| \V{c}^* \right) \rightarrow 0.
\end{align*}
\end{proposition}
Let $\hat{\V{q}}^{\V{c}}$ be the maximizer of $\hat{J}(q_1(\V{c}), q_2(\V{\alpha}),\V{\pi},\M{\Omega})$ over $q_1(\V{c})$.
Proposition \ref{thm:separation} and Proposition \ref{thm:uniform} imply the  label consistency:
\begin{theorem}\label{thm:consistency}
Under the assumptions of Proposition \ref{thm:separation} and Proposition \ref{thm:uniform}, as $N,P \rightarrow \infty$ and $\log(P)/N = o(1)$, for all $\epsilon >0$,
	\begin{align*}
	& \mathbb{P} \left( \left. \min_{\tilde{\V{I}}^c \in \mathcal{E}_{\V{I}^c} } \left (1-\textnormal{Tr} (\M{R}(\hat{\V{q}}^c,\tilde{\V{I}}^c))\right ) \geq \epsilon \right |\V{c}^* \right) \rightarrow 0.
	\end{align*}
\end{theorem}




\section{Simulation Studies}

In this section, we conduct simulations to assess the performance of the proposed HBCM in clustering and to compare it with spectral clustering. Among the various clustering methods, spectral clustering and K-means stand out as the most popular. However, the latter does not align with our model setup, thus we solely present a comparison between HBCM and spectral clustering. We gauge the clustering performance using the adjusted Rand index (ARI) \citep{rand}, which quantifies the level of agreement between the estimated group labels and the ground truth. The ARI score ranges from zero (random guess) to one (complete agreement), with a higher value indicating greater agreement. We examine multiple simulation scenarios with various parameter configurations and introduce a cross-validation approach to determine the number of communities (denoted as $K$) when it is unknown.

\subsection{Numerical experiments with various sample sizes and dimensionality}

We begin by examining how clustering performance varies with different values of $N$, $P$, and $K$. The simulated data are generated as follows:
\begin{enumerate}
    \item Set $N = 500$ or $N = 1000$. For $N = 500$, use $P = 300$, $500$, and $1000$. For $N = 1000$, use $P = 500$, $1000$, and $1500$. For each setup, set $K = 3$, $5$, and $7$.
    \item Set the community-level covariance matrix $\bome_{K \times K}$ to have diagonal elements equal to 1 and off-diagonal values equal to 0.5. The true label vector $\V{c}$ is generated by $P$ independent draws from Multinomial$(1,\V{\pi})$ with $\V{\pi}=(1/K,...,1/K)$.
    \item For $j=1,...,P$, generate ${\lambda_j}$ from normal distribution $N(0,1)$ and generate ${\sigma_j^2}$ from a chi-square distribution with degree of freedom 2. Add 1 to ${\sigma_j^2}$ to ensure it remains non-zero.
    \item For $i=1,...,N$, generate $\alpha_i$ from $K$-variate normal distribution $N(\V{0}, \bome)$. For each $i,j$, generate $X_{ij}$ by $X_{ij}=\lambda_j \alpha_{ic_j}+ \sigma_j \epsilon_{ij}$, where $\epsilon_{ij}\,\, \text{i.i.d.} \sim N(0,1)$.
\end{enumerate}

Under each $(N, P, K)$ configuration, simulations are replicated 100 times. For each generated dataset, both the HBCM and spectral clustering methods are employed to estimate the the group membership. The average ARI and empirical standard deviations are presented in Table~\ref{tab:ARI}. 
We adopt the absolute sample correlation matrix as the kernel matrix in spectral clustering.

\begin{table}[h!]
\centering
\begin{tabular}{|l|ll|ll|ll|}
     \hline
& \multicolumn{2}{c|}{$P$=300} & \multicolumn{2}{c|}{$P$=500} & \multicolumn{2}{c|} {$P$=1000} \\
 \hline
 & \multicolumn{6}{c|}{$N$=500}\\
 \hline
$K$ & Spectral & HBCM & Spectral & HBCM & Spectral & HBCM  \\
\hline
3 & 0.26 (0.03) & 0.46 (0.14) & 0.25 (0.04) & 0.49 (0.15) & 0.25 (0.07) & 0.49 (0.14) \\
5&  0.38 (0.04) & 0.45 (0.09) & 0.36 (0.03) & 0.46 (0.08) & 0.35 (0.02) & 0.49 (0.06)\\
7& 0.41 (0.04) & 0.43 (0.09) & 0.39 (0.03) & 0.46 (0.12) & 0.38 (0.02) & 0.49 (0.05)\\
\hline
& \multicolumn{2}{c|}{$P$=500} & \multicolumn{2}{c|}{$P$=1000} & \multicolumn{2}{c|} {$P$=1500} \\
 \hline
 & \multicolumn{6}{c|}{$N$=1000}\\
 \hline
$K$ & Spectral & HBCM & Spectral & HBCM & Spectral & HBCM  \\
\hline
3 & 0.31 (0.07) & 0.52 (0.17) & 0.36 (0.13) & 0.60 (0.17) & 0.37 (0.14) & 0.61 (0.16)\\
5 & 0.44 (0.03) & 0.52 (0.12) & 0.40 (0.02) & 0.53 (0.08) & 0.39 (0.02) & 0.53 (0.05) \\
7 & 0.48 (0.03) & 0.57 (0.04) & 0.44 (0.02) & 0.56 (0.05) & 0.43 (0.02) & 0.57 (0.05)\\
\hline
\end{tabular}
\caption{Evaluate and compare the clustering performance of HBCM and spectral clustering under various setups of $N,P,K$, using the adjusted rand index.}\label{tab:ARI}
\end{table}

In Table \ref{tab:ARI}, the HBCM displays better performance compared to spectral clustering, with higher ARI scores across all settings. As $N$ increases, the ARI for both methods generally rises. Additionally, although HBCM uses spectral clustering in initialization, its results can substantially improve beyond those of spectral clustering, especially when the spectral clustering performance is suboptimal.

\subsection{Numerical experiments with various parameters}

We investigate the impact of signal strength, quantified by $\lambda_j/\sigma_j$, as well as the ratio of inter-community covariance to intra-community covariance ($\omega_{kl}/\omega_{kk}$), on the clustering performance of both HBCM and spectral clustering. We consider simulation setups with $N=1000$, $P=1000$, and $K=3$. In Figure~\ref{simParameter}, the left panel plot illustrates the results of the following simulation setup: we set the heterogeneous parameter $\lambda_j$ to be constant across columns and vary the parameter $\sigma_j$ from 1 to 10. This plot reveals that the performance of both HBCM and spectral clustering diminishes as the signal-to-noise ratio decreases. As the ratio becomes lower, it becomes more challenging for both models to discern the underlying cluster patterns. When we maintain the parameter $\lambda_j$ at 1, $\sigma_j$ at 6, and the diagonal $\omega_{kk}$ of the covariance matrix at 1, while progressively increasing the off-diagonal $\omega_{kl}$ from 0.1 to 0.9, the middle plot indicates a decline in performance for both methods as the ratio of inter-community covariance to intra-community covariance approaches 1.

Furthermore, we assess the robustness of our approach when the normal distribution is substituted with a $t$-distribution of varying degrees of freedom for the error term $\epsilon_{ij}$ in Section 1.4.1 (Step 4). The $t$-distribution with $v$ degrees of freedom ($v>2$) has a variance of $v/(v-2)$, which exceeds one. We consider two data generation processes: $\epsilon_{ij}$ is drawn from a $t$-distribution with $v$ degrees of freedom, and $\epsilon_{ij}$ from the $t$-distribution is re-standardized by dividing it by the standard deviation $\sqrt{v/(v-2)}$. In the right panel of Figure~\ref{simParameter}, solid lines and dashed lines correspond to scenarios without and with variance correction, respectively. In general, the plot indicates that the performance of both methods is significantly influenced by the variance of $\epsilon_{ij}$ -- specifically, their performance deteriorates with larger $v$ values. However, the performance is less susceptible to the distributional shape of $\epsilon_{ij}$ -- that is, the ARI remains nearly constant when the variance has been adjusted.

In summary, as anticipated, the clustering performance of both models hinges on the signal strength. Regardless of the simulation scenarios, HBCM (indicated by the blue line) consistently outperforms spectral clustering (represented by the orange line).

\begin{figure}
\begin{center}
\includegraphics[width=6in]{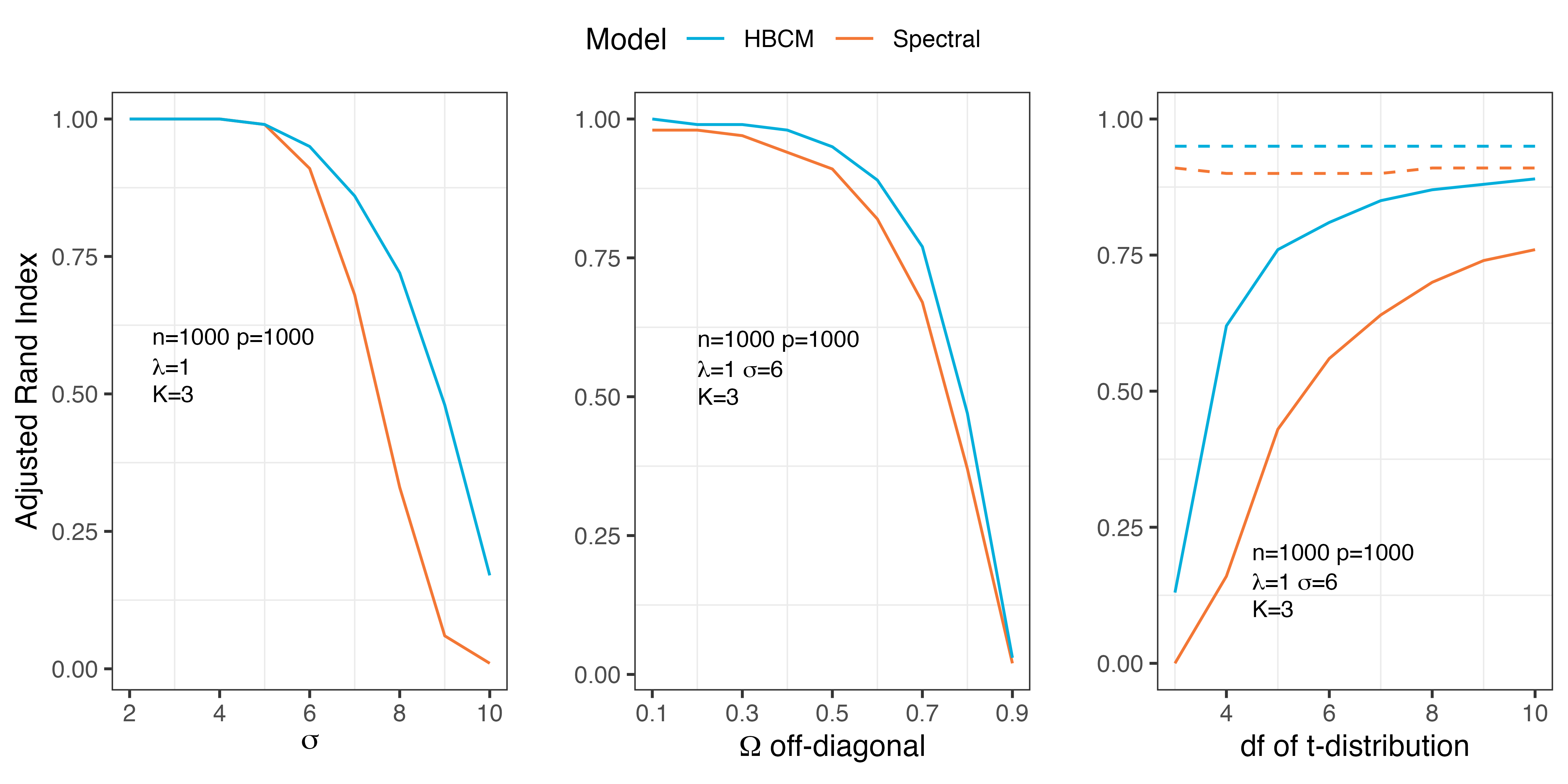}
\end{center}
\vspace{-1cm}
\caption{Compare the clustering performance of HBCM and spectral clustering under various parameter setups. In the right plot, dashed line represents simulation data with variance correction and solid line indicates data without any variance correction.}
\label{simParameter}
\end{figure}

Next, we conduct an additional set of simulations to gain further insights into the factors contributing to HBCM's superior performance over spectral clustering. We consider a setup where $N=1000$, $P=1000$, $K=3$, $\omega_{kk}=1$, $\omega_{kl}=0.5$, and $\lambda_j\in \{1,5,25\}$. The heterogeneous parameters ${\lambda_j}$ take three distinct values: the first 330 $\lambda_j$ values ($j=1,\ldots,330$) are set to 1, the next 330 $\lambda_j$ values ($j=331,\ldots,660$) are set to 5, and the last 340 $\lambda_j$ values ($j=661,\ldots,1000$) are set to 25.
Up to this point, we have considered three sets of labels in the simulations: 1) the true label $L_{true}$, 2) the estimated spectral clustering label $L_{spectral}$, and 3) the estimated HBCM label $L_{HBCM}$. With this new simulation setup, we introduce a fourth set of labels, $L_{mislead}$, based on the heterogeneous $\lambda_j$ values. Specifically, columns that share the same $\lambda_j$ are assigned to the same group.

In Figure~\ref{lambdaMargin}, we plot the ARI against different values of $\sigma_j$. The dashed line indicates the ARI scores between the label $L_{mislead}$ and the estimated labels, while the solid line represents the ARI values between the true label $L_{true}$ and the estimated labels. Ideally, a reliable method should yield label estimates similar to the true labels (high ARI for the solid lines) while remaining unaffected by the presence of heterogeneous $\lambda_j$ values in the misleading labels (low ARI for the dashed lines). It can be observed that as the signal-to-noise ratio decreases, the communities detected by spectral clustering become more aligned with the $L_{mislead}$ label than with the true label. Conversely, the HBCM exhibits no agreement with the $L_{mislead}$ label and demonstrates stronger alignment with the true label. Upon closer inspection of the 1000 labels estimated by spectral clustering, the first 330 columns are indeed categorized into a single group corresponding to $\lambda_j=1$, while the HBCM still distinguishes them into three groups. This observation implies that in scenarios with low signal-to-noise ratios, spectral clustering tends to be influenced by ${\lambda_j}$. By contrast,  the HBCM retains the ability to differentiate between individual parameters ${\lambda_j}$ and group signals $\bome$, resulting in improved clustering performance.

\begin{figure}
\begin{center}
\includegraphics[width=5in]{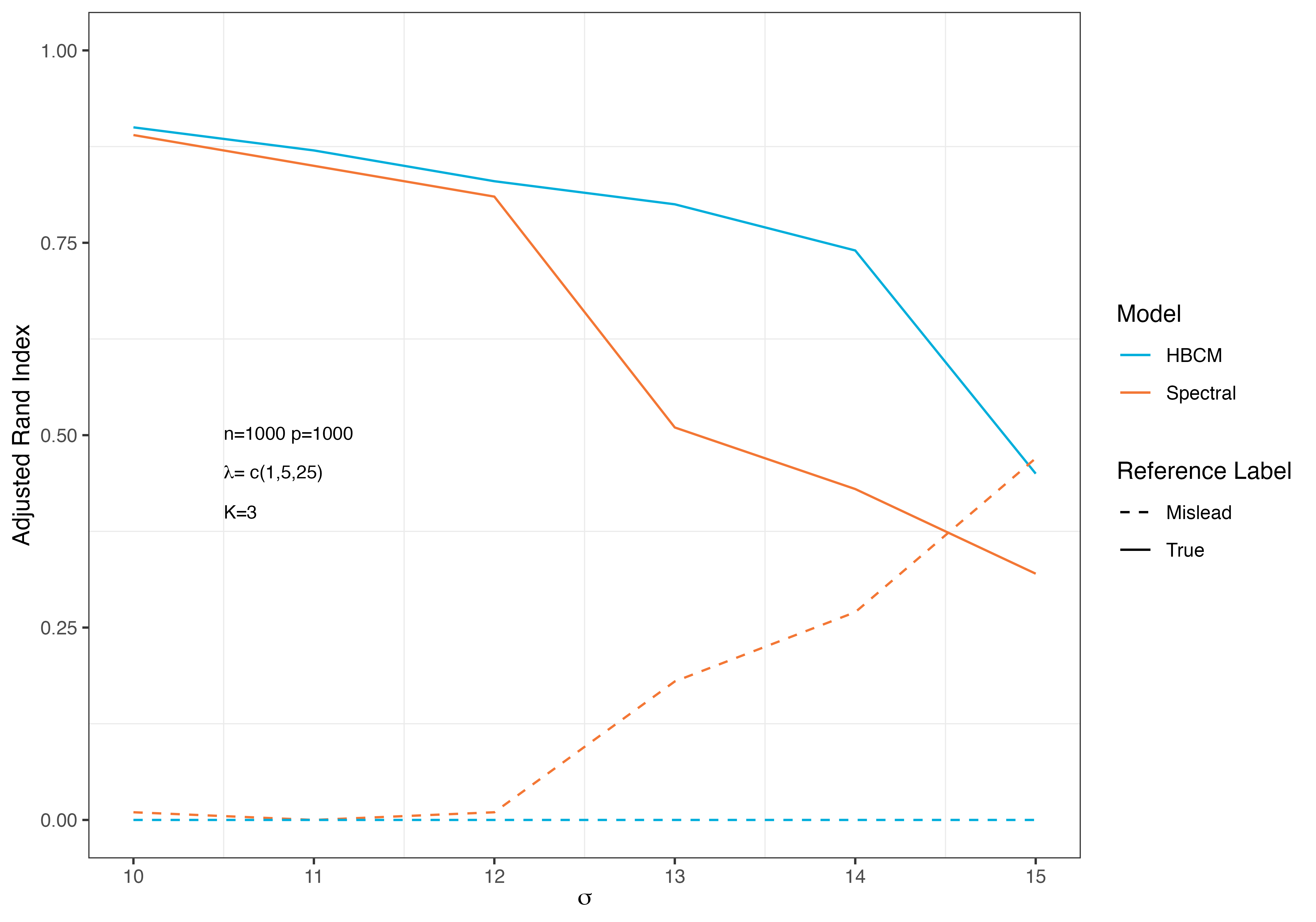}
\end{center}
\vspace{-1cm}
\caption{Spectral clustering is misled by the individual signal $\lambda_j$ as the signal-to-noise ratio decreases. }
\label{lambdaMargin}
\end{figure}

 Finally, we conduct a simulation study to investigate the robustness of HBCM when the population covariance matrix is misspecified. Let $\tilde{\M{\Omega}} = [\tilde{\omega}_{jj'}]$ be a $P \times P$ block matrix generated from $\M{\Omega}$, where $\tilde{\omega}_{jj'} = \omega_{c_j c_{j'}}$.  We add a semi-definite covariance matrix $\M{W} = (1/\sqrt{10}) \sum_{m=1}^{10} \V{z}_m \V{z}_m^\top$ to $\tilde{\M{\Omega}}$, where each $\V{z}_m$ is a column vector with entries independently following a standard normal distribution. It is easy to verify that the off-diagonal entries of $\M{W}$ have a mean of 0 and a variance of 1.  Let $\tilde{\M{\Sigma}} = \textnormal{diag}(\V{\lambda})( \tilde{\M{\Omega}}+r \M{W})\textnormal{diag}(\V{\lambda}) + \textnormal{diag}(\V{\sigma}^2)$. We use $r$ to control the deviation between $\tilde{\M{\Sigma}}$ and the population covariance matrix $\M{\Sigma}$ in the original setup. In this simulation, we set $r$ from 0 to 0.5 in increments of 0.1. The remaining settings for the simulation are as follows: $n=1000$, $p=1000$, $K=3$, and $\bome$ has diagonal elements equal to 1 and off-diagonal elements equal to 0.5. The cluster labels are randomly assigned to three clusters with equal probabilities. The parameters $\lambda_j$ and $\sigma_j$ are fixed at 1 and 6, respectively.

\begin{figure}[!ht]
\begin{center}
\includegraphics[width=3.5in]{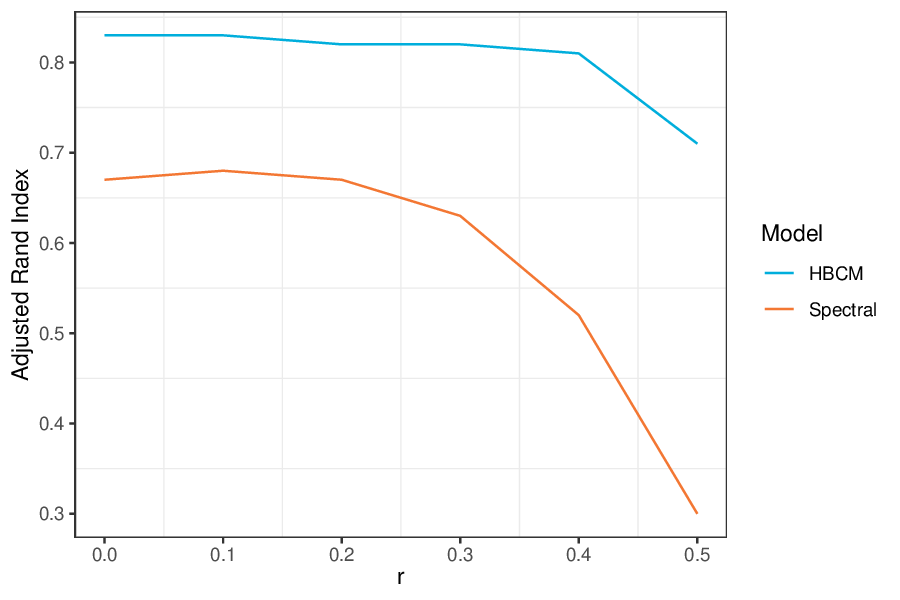}
\end{center}
\vspace{-1cm}
\caption{Compare the clustering performance of HBCM and spectral clustering in the presence of noisy population covariance matrices.}
\label{fig:robust_covariance}
\end{figure}

Figure \ref{fig:robust_covariance} demonstrates that HBCM consistently outperforms spectral clustering for all values of $r$. Moreover, it is interesting to note that HBCM, as a likelihood-based approach, exhibits greater robustness than spectral clustering as the population covariance matrix gradually deviates from the block matrix. It is intriguing to study the theoretical properties of HBCM under misspecified model setups for future work.

\subsection{Cross-validation for selecting the  number of communities} \label{sec:cv}

The number of communities is usually a required input for clustering methods, but it is often unknown in practice. We propose a cross-validation-based method to select the number of communities and evaluate its performance on simulated data. The detailed steps are outlined in Algorithm~\ref{algoCV}. The proposed method resembles a standard 2-fold cross-validation, using ARI as the objective function. However, it differs in that we do not designate one fold of sub-data as training and the other as test data. Instead, we treat both folds of sub-data equally. We apply HBCM to both sets and choose the number of communities that yields the highest agreement between the two sets of estimated labels, as measured by ARI. The rationale is that if the specified number of communities aligns with the truth, the two sets of labels should exhibit greater consistency, resulting in a higher ARI. Readers may find further discussion on clustering stability in \cite{von2010clustering}.

We examine the performance of the proposed cross-validation-based method through a simulation. We set $N=1500$, $P=500$, $K=5$, $\omega_{kk}=1$, $\omega_{kl}=0.2 \,\,(k\neq l)$, $\lambda_j=1$, and $\sigma_j=6$. For each potential number of communities $K$ ranging from 2 to 9, we set $M=20$ and perform Algorithm~\ref{algoCV}. In the left panel of Figure~\ref{KbyCV}, the proposed cross-validation method successfully identifies the true number of communities as $K=5$ with the maximum ARI for the simulated data.

\begin{algorithm}[H]
\caption{Cross-Validation for Selecting $K$}
\label{algoCV}
\textbf{Input:} A sequence of candidate $K$; 

\textbf{Step 1:} For each $K$ in the sequence and for $m = 1,..., M$:
    \begin{enumerate}
        \item Randomly split the data into two parts with equal number of rows.
        \item Apply the clustering method on each subset and obtain the label estimation $\V{c}_1^{(K,m)}, \V{c}_2^{(K,m)}$ respectively.
        \item Calculate the ARI$^{(K,m)}$ based on the degree of agreement between $\V{c}_1^{(K,m)}, \V{c}_2^{(K,m)}$.
    \end{enumerate}

\textbf{Step 2:} For each $K$, calculate the average ARI$^K$ as $\sum_{m=1}^M\text{ARI}^{(K,m)}/M$;

\textbf{Step 3:} Make a line plot of average ARI against the sequence of candidate $K$.

\textbf{Output:} Choose $K$ with the maximum ARI score as the number of communities.
\end{algorithm}

\begin{figure}
\begin{center}
\includegraphics[width=3in]{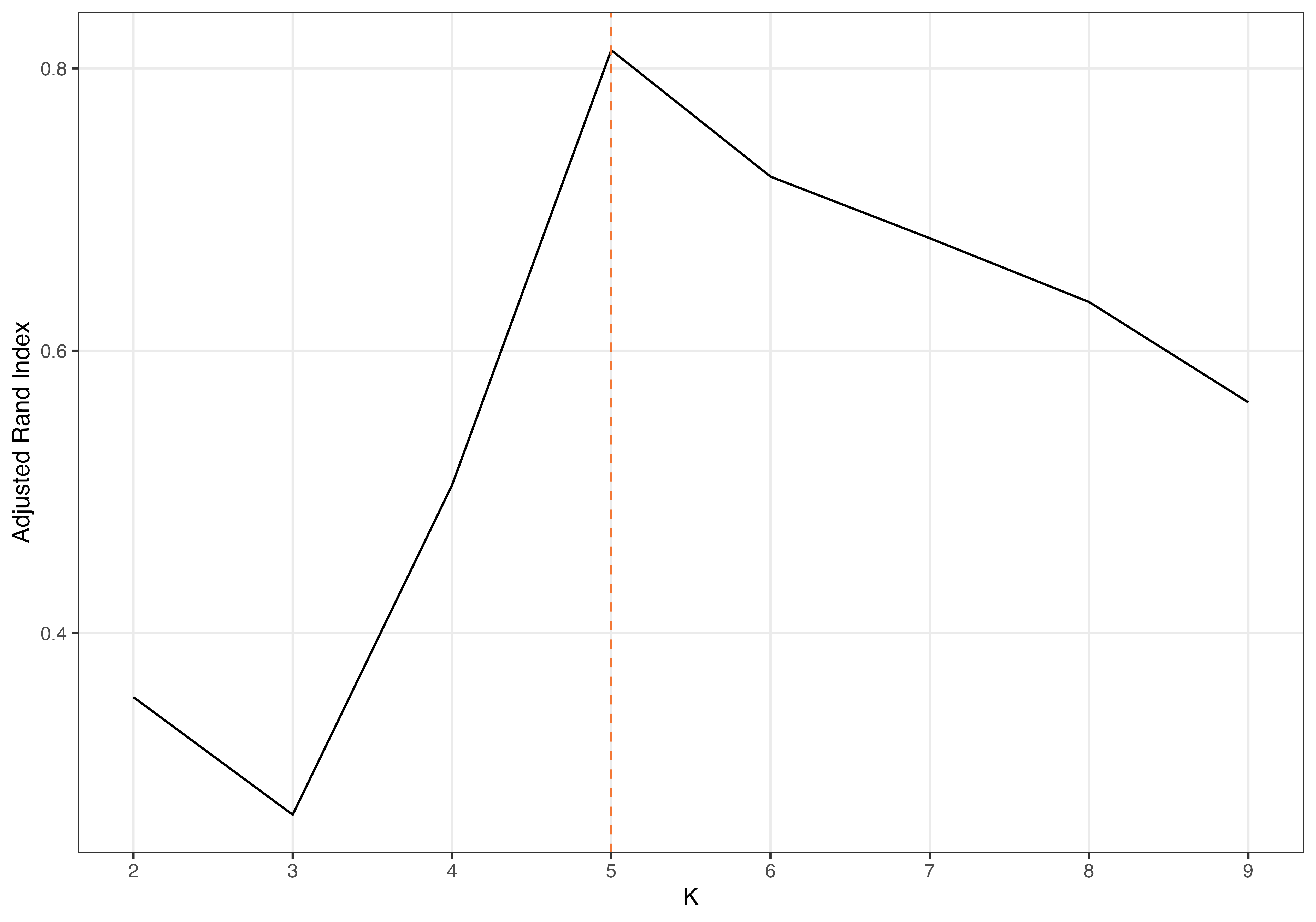}
\end{center}
\vspace{-1cm}
\caption{Cross-validation method for selecting the number of communities. }
\label{KbyCV}
\end{figure}

\section{Application on Mouse Embryo Single-Cell RNA-seq Data}

We analyze single-cell RNA (scRNA) sequencing data from individual mouse embryonic stem (ES) cells, as described in \cite{islam2011characterization}. Single-cell scRNA sequencing is a high-throughput technique that quantifies gene-specific transcript counts at the level of individual cells. To enhance data quality, genes that are undetected in more than 40 of the 92 cells are excluded. After preprocessing, the dataset includes 1423 genes, resulting in a data matrix with $N = 92$ and $P = 1423$, following the notation used in this paper.  Raw data counts were transformed using log-counts-per-million (logCPM), which scales raw counts by converting them to counts per million, followed by a log transformation. 

\begin{figure}[!ht]
	\twoImages{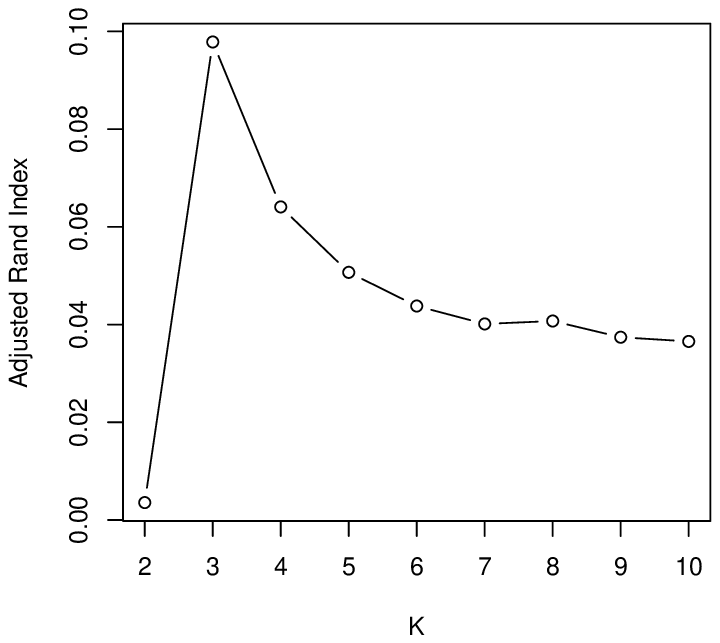}{6cm}{(a) Selecting the number of communities}{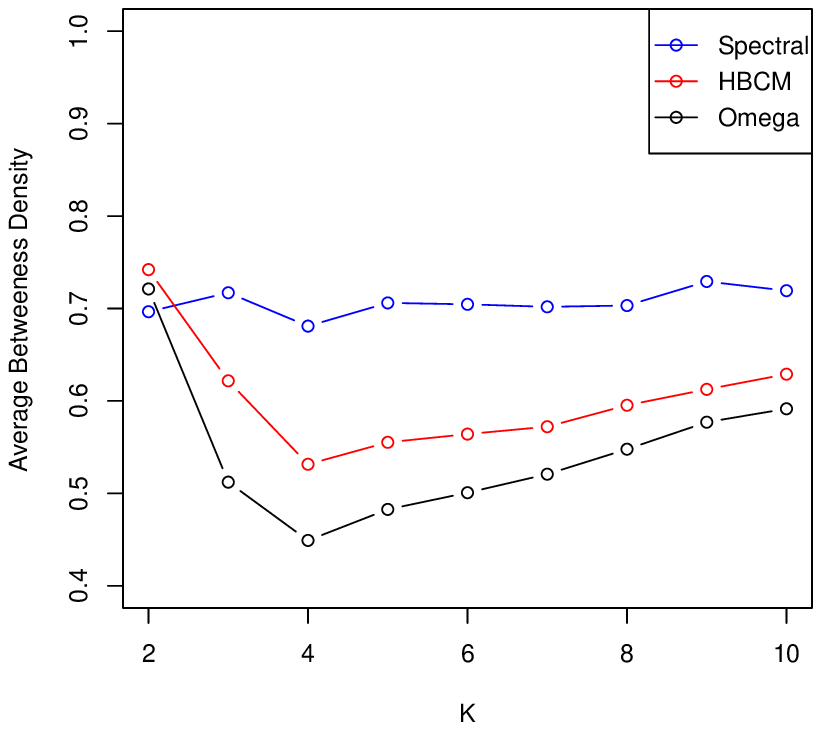}{6cm}{(b) Between-cluster densities}
	\caption{Left panel: Cross-validation method for selecting the number of communities in the scRNA data. Right panel: between-cluster densities of the absolute correlation matrix and $\hat{\M{\Omega}}$. }
	\label{fig:mouse1}
\end{figure}
\begin{figure}[!ht]
	\twoImages{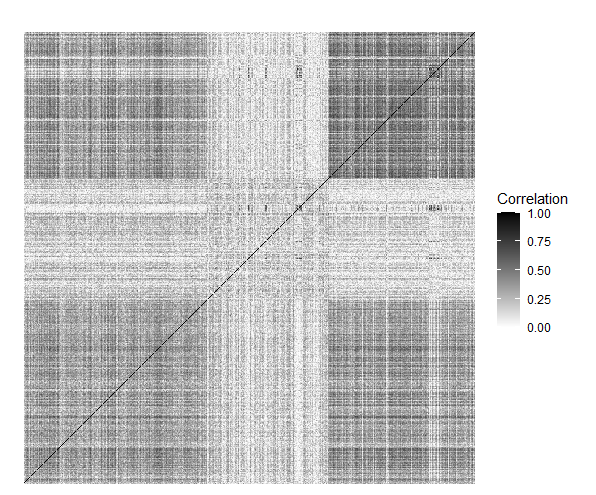}{6cm}{(a) Spectral clustering}{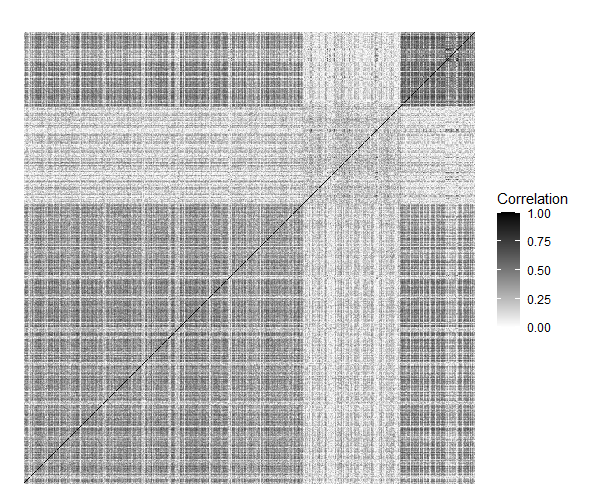}{6cm}{(b) HBCM }
	\caption{Heatmaps of the absolute correlation matrix.}
	\label{fig:mouse_heatmap}
\end{figure}
First, we apply the cross-validation method proposed in Section \ref{sec:cv}
 to select the number of communities, using $M=100$. As shown in the left panel of Figure \ref{fig:mouse_heatmap}, the maximum ARI score is achieved at $K=3$. Next, we apply spectral clustering and HBCM with $K=3$ to the dataset. Figure \ref{fig:mouse_heatmap} displays heatmaps of the absolute correlation matrix, with genes reordered according to the clustering results from each method, respectively. To further compare the two methods, we compute the average values within each block of the absolute correlation matrix, forming a $3 \times 3$ matrix and excluding diagonal entries (which are always 1). The resulting matrices for spectral clustering and HBCM are: 
 \begin{align*}
B_{\textnormal{Spectral}} = \begin{pmatrix}
 0.33 & 0.18 & 0.34 \\
 0.18 & 0.18 & 0.15 \\
 0.34 & 0.15 & 0.43
\end{pmatrix}, \quad 
B_{\textnormal{HBCM}} = \begin{pmatrix}
0.33 & 0.17 & 0.31 \\
0.17 & 0.20 & 0.13 \\
0.31 & 0.13 & 0.42
\end{pmatrix}.
 \end{align*}
From the two matrices, HBCM demonstrates a stronger contrast between between-cluster and within-cluster densities compared to spectral clustering. In $B_{\textnormal{Spectral}}$, the density between clusters 1 and 3 is higher than the within-cluster density of cluster 1, and the density between clusters 1 and 2 is identical to the within-cluster density of cluster 2. On the contrary, $B_{\textnormal{HBCM}}$ presents a clearer contrast. To systematically compare the between-cluster and within-cluster densities, we further summarize matrix $B$ into a single value by standardizing $B$ and averaging its off-diagonal elements. Specifically, let $B_{\textnormal{standardized}} = \diag \left ( \sqrt{B_{11}^{-1}},\dots,\sqrt{B_{KK}^{-1}} \right ) \,\, B \,\, \diag \left ( \sqrt{B_{11}^{-1}},\dots,\sqrt{B_{KK}^{-1}} \right )$ and average its off-diagonal element. In addition, we include the between-cluster and within-cluster densities of $\hat{\M{\Omega}}$ in the comparison. Recall that $\hat{\M{\Omega}}$ characterizes the covariance structure at the community level (eliminating the effect of $\V{\lambda}$) and is identifiable up to multiplication by diagonal matrices on the left and right. Therefore, we standardize $\hat{\M{\Omega}}$ and take the average of the absolute values of its off-diagonal elements as before. These comparisons for $K=2, \dots, 10$ are illustrated in the right panel of Figure \ref{fig:mouse1}. Except for $K=2$, a clear pattern emerges: the between-cluster density of $B_{\textnormal{standardized,Spectral}}$ is greater than that of $B_{\textnormal{standardized,HBCM}}$, which in turn is greater than that of $\hat{\M{\Omega}}$. This indicates that HBCM consistently demonstrates a stronger contrast between between-cluster and within-cluster densities, while $\hat{\M{\Omega}}$ shows an even clearer contrast when the effect of $\lambda$ is isolated. Furthermore, results regarding WGCNA and gene set enrichment analysis (GSEA) are provided in the Supplemental Materials.

\section{Discussion}

We have introduced a community detection model that is designed for weighted networks with continuous links between nodes, while simultaneously accommodating the heterogeneity of individual nodes within the same community. The proposed HBCM can be directly applied to the raw data matrix, where rows represent samples and columns represent genes, rather than relying on a constructed adjacency matrix. To estimate the group membership, we have introduced a novel variational EM algorithm by incorporating a second layer of latent variables and a variational approximation, resulting in a significant reduction in computational complexity. We have established that the group membership estimated by the proposed model are consistent with the true labels under mild conditions. Through extensive simulation studies, we have demonstrated that the proposed model outperforms traditional spectral clustering, even in scenarios where the population covariance structure is misspecified. HBCM also outperforms classical community detection methods including SBM and DC-SBM, when applied to a thresholded matrix (see Section 1 in the Supplementary Materials). Additionally, we have introduced a cross-validation-based method for selecting the number of communities when it is unknown.

We plan to explore several directions for future work. Firstly, our theoretical results assume a fixed number of communities. It is intriguing to investigate the asymptotic behavior of the model as the number of communities increases. Secondly, the HBCM can be extended to encompass a biclustering scenario, wherein both rows and columns of the data matrix are clustered simultaneously. Thirdly, our aim is to develop a more systematic criterion for selecting the number of communities, ideally with potential theoretical justification. Lastly, our analysis is based on the covariance structure of the data matrix, which aligns with typical module detection methods for gene expression data \citep{saelens2018comprehensive}. The mean structure could likewise potentially impact cluster analysis. It would be intriguing to explore a simultaneous analysis of both mean and covariance structures in future work.


\section*{Supplemental Materials}
\addcontentsline{toc}{section}{Supplemental Materials} 
\startsupplement

\section{Additional simulation studies}

\subsection{Comparison of HBCM with SBM and DC-SBM}
An alternative approach to conducting cluster analysis on the covariance or correlation matrix is to transform the matrix into a network with binary edges and apply classical community detection models, such as the stochastic block model (SBM) or the degree-corrected stochastic block model (DC-SBM).
We conduct an additional simulation study using the settings from Table 1 of the main text to compare the clustering results of HBCM with those of the SBM and DC-SBM.

First, we calculate the absolute sample correlation matrix and dichotomize it using the mean values as the threshold. The resulting networks are denser than those typically studied in the community detection literature, as our aim is to retain sufficient information for subsequent analyses. We then fit the dichotomized networks using the pseudo-likelihood method \citep{amini2013pseudo} for both SBM and DC-SBM. We use spectral clustering on the dichotomized networks to provide initial values for both methods.

\begin{table}[h!]
\centering
\begin{adjustbox}{width=1.0\textwidth, center=\textwidth }
\begin{tabular}{|l|lll|lll|lll|}
     \hline
& \multicolumn{3}{c|}{$P$=300} & \multicolumn{3}{c|}{$P$=500} & \multicolumn{3}{c|} {$P$=1000} \\
 \hline
 & \multicolumn{9}{c|}{$N$=500}\\
 \hline
$K$  & HBCM & SBM & DCSBM  & HBCM & SBM & DCSBM  & HBCM & SBM & DCSBM  \\
\hline
3  & 0.46 (0.14) & 0.00 (0.01) & 0.30 (0.14)
 & 0.49 (0.15) & 0.00 (0.00) & 0.34 (0.07)
 & 0.49 (0.14)  & 0.00 (0.00) & 0.36 (0.07)\\
5&  0.45 (0.09) & 0.00 (0.01) & 0.11 (0.07)
&  0.46 (0.08) & 0.00 (0.01) & 0.23 (0.12)
&  0.49 (0.06) & 0.01 (0.01) & 0.22 (0.10)\\
7& 0.43 (0.09) & 0.00 (0.01) & 0.04 (0.03)
& 0.46 (0.12) & 0.00 (0.01) & 0.10 (0.04)
& 0.49 (0.05) & 0.00 (0.01) & 0.14 (0.03)\\
\hline
& \multicolumn{3}{c|}{$P$=500} & \multicolumn{3}{c|}{$P$=1000} & \multicolumn{3}{c|} {$P$=1500} \\
 \hline
 & \multicolumn{9}{c|}{$N$=1000}\\
 \hline
$K$  & HBCM & SBM & DCSBM & HBCM & SBM & DCSBM  & HBCM & SBM & DCSBM  \\
\hline
3  & 0.52 (0.17) & 0.00 (0.00) & 0.12 (0.18)
 & 0.60 (0.17) & 0.00 (0.00) & 0.41 (0.21)
 & 0.61 (0.16) & 0.00 (0.00) & 0.53 (0.23)\\
5  & 0.52 (0.12) & 0.00 (0.01) & 0.06 (0.08)
 & 0.53 (0.08) & 0.01 (0.01) & 0.17 (0.07)
 & 0.53 (0.05) & 0.01 (0.01) & 0.20 (0.03)\\
7  & 0.57 (0.04) & 0.00 (0.00) & 0.01 (0.03) 
 & 0.56 (0.05) & 0.00 (0.00) & 0.09 (0.07)
 & 0.57 (0.05) & 0.00 (0.00) & 0.19 (0.05)\\
\hline
\end{tabular}
\end{adjustbox}
\caption{Evaluate and compare the clustering performance of HBCM, SBM, and DCSBM under various setups of $N,P,K$, using the adjusted rand index.}\label{tab:ARI_SBM}
\end{table}

From Table \ref{tab:ARI_SBM}, the SBM consistently yields ARI values close to 0, indicating that it does not adequately account for degree variation. In contrast, the DC-SBM produces meaningful ARI values, although they are consistently lower than those of the HBCM due to information loss from dichotomization. In particular, the performance of the DC-SBM substantially deteriorates as the number of clusters increases.

\subsection{Comparison of HBCM with WGCNA}
We conduct a simulation study using one setting from Table 1, i.e., $N=500$, $P=300$, and $K=3$, to illustrate the clustering performance of WGCNA. We use the widely applied R package \textsf{WGCNA} to perform the clustering. One caveat of using the package is that there are many customizable parameters that can impact the estimated number of clusters. We select the parameter values to increase the likelihood of clustering the columns into $K=3$ groups. Specifically, we use the power $\beta = 6$ when constructing the weighted adjacency matrix by raising the absolute correlation to a power. The simulation shows that the average ARI of the WGCNA clustering results is 0.01, which is worse than the ARI of Spectral Clustering (0.26) and HBCM (0.48). The main reason this method performs poorly in our scenario is that it relies on hierarchical clustering, which tends to produce a ``core-periphery'' structure in the dendrogram, differing from our assumptions.

\subsection{HBCM with random initial cluster assignments}
We use spectral clustering results as initial values when fitting HBCM in the main text. In this section, we investigate the performance of HBCM with random initial cluster assignments. We conduct a simulation study using the following setup: let $N=1000$, $P=1000$, $K=2$, and $\bome$ have diagonal elements equal to 1 and off-diagonal elements equal to 0.5. The cluster labels are randomly assigned to the two clusters with equal probabilities. Furthermore, $\lambda_j$ for $j=1,\dots,500$ is set to 1, for $j=501,\dots,1000$ it is set to -1, and $\sigma_j$ is assigned a constant value of 8. The simulation is repeated 100 times. For each dataset, we compare three methods: spectral clustering, HBCM with spectral clustering as initial assignments, and HBCM with random initial assignments. For HBCM with random initial assignments, we randomly generate 10 initial values and select the solution that gives the highest value of the objective function $J(q(\V{c},\V{\alpha}), \Phi)$. Figure \ref{HBCM_random} shows the distributions of adjusted Rand indices for the three methods. The figure demonstrates that HBCM with spectral clustering as initial assignments and HBCM with random initial assignments perform similarly, both generally outperforming spectral clustering. However, HBCM with random initial assignments occasionally generates poor clustering results (see outliers in Figure \ref{HBCM_random}), indicating that this method is less stable than when using spectral clustering as initial values. Moreover, HBCM with random initial values takes more iterations to converge. The average number of iterations required by HBCM with spectral clustering as initial values and by HBCM with random initial values is 12.69 and 21.12, respectively. 

\begin{figure}[!ht]
\begin{center}
\includegraphics[width=5in]{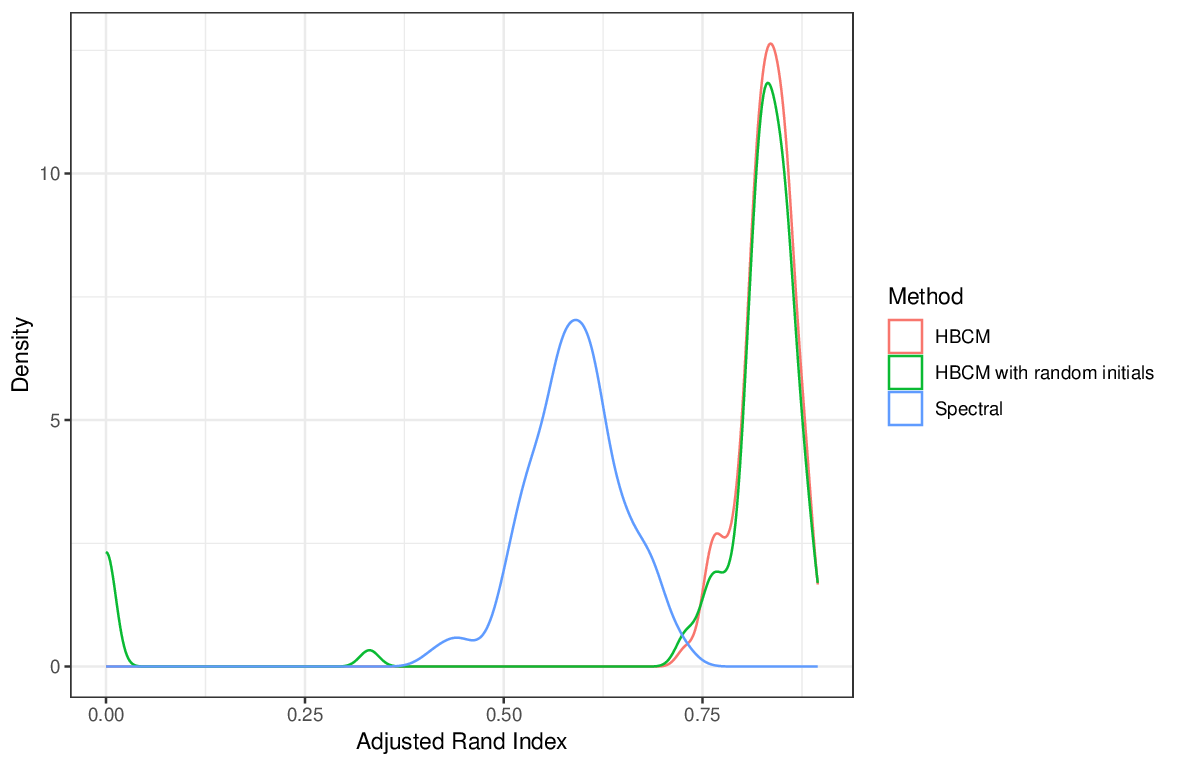}
\end{center}

\caption{Comparison of spectral clustering, HBCM, and HBCM with random initials}
\label{HBCM_random}
\end{figure}

\subsection{Analysis of computation complexity}

We conduct additional simulation studies to assess the computational complexity of the proposed method with respect to $N$ and $P$. In the first study, we fix $P = 1000$ and $K = 3$, and vary $N$ from 500 to 1500 with increments of 200. In the second study, we fix $P = 1000$ and $K = 3$, and vary $P$ from 500 to 1500 with increments of 200. The rest of the setting is the same as in Table 1 in the main text. The simulations were conducted on a research computing cluster with AMD EPYC CPUs. The CPU runtimes (in seconds) are reported in Figure \ref{fig:time}. All simulations are completed within one minute. As shown in the figure,  the algorithm's complexity is nearly linear with respect to both $N$ and $P$, which demonstrates that the proposed algorithm is not computationally demanding. Furthermore, it can be easily verified that the complexity of the E and M steps is $O(NP)$ according to  Algorithm 1 and Eq. (5) in the main text.

\begin{figure}[!ht]
	\twoImages{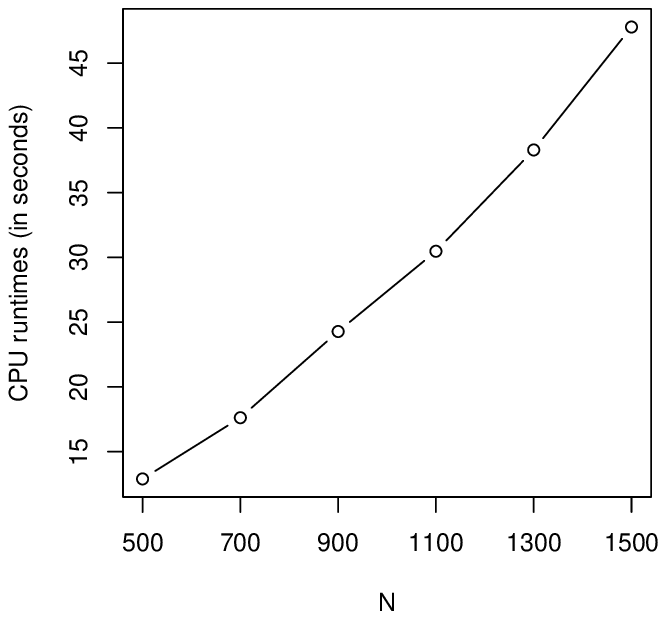}{6cm}{(a) $N$ varies}{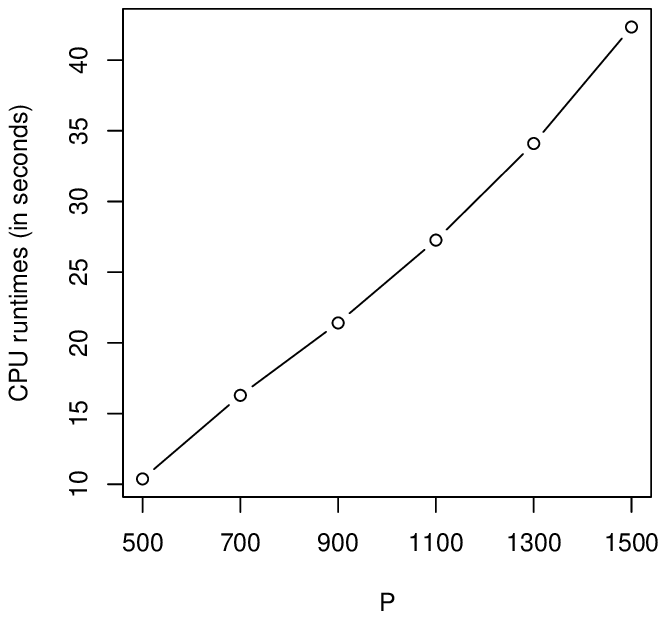}{6cm}{(b) $P$ varies}
	\caption{CPU runtimes (in seconds) of HBCM as $N$ and $P$ vary.}
	\label{fig:time}
\end{figure}

\section{Analysis of stock price data }

We analyze the stock price data described in \cite{liu2012nonparanormal} and \cite{tan2015cluster}. This dataset, available in the R package \textsf{HUGE}, contains daily closing prices for stocks in the S\&P 500 index from January 1, 2003, to January 1, 2008. The stocks are categorized into 10 Global Industry Classification Standard (GICS) sectors: industrials, financials, health care, consumer discretionary, information technology, utilities, materials, consumer staples, telecommunications services, and energy. Following the approach of \cite{tan2015cluster}, we exclude any stocks that were not consistently included in the S\&P 500 index during this period. This results in a total of 1258 daily closing prices for 452 stocks. To reduce the temporal dependency on stock prices over days, we consider the difference between stock prices on consecutive days. Specifically, let $y_{ij}$ be the closing price of the $j$-th stock on the $i$-th day, and $\M{X} = [x_{ij}]$ be a 1257-by-452 data matrix with $x_{ij} = y_{i+1,j} - y_{ij}$. 

To systematically compare HBCM, spectral clustering, and WGCNA, we created submatrices of $\M{X}$ by taking the first $N$ rows, where $N=100, 200, \ldots, 1200, 1257$. Following the standard protocol in this paper, we centralize each submatrix so that each stock has a mean of zero. We perform HBCM, spectral clustering, and WGCNA on the submatrices and report the ARIs by comparing the estimated labels with the true GICS categories in Figure \ref{fig:stock}. We use $K = 10$ for both HBCM and spectral clustering. For WGCNA, we select the tuning parameters to likely partition the data into 10 groups. 

\begin{figure}[!ht]
	\twoImages{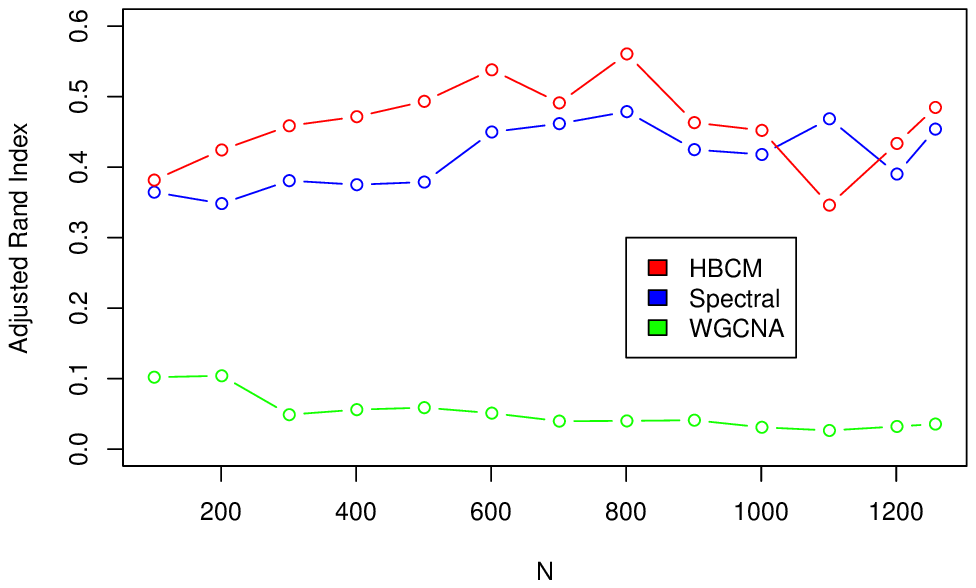}{6cm}{(a) ARI against the true labels}{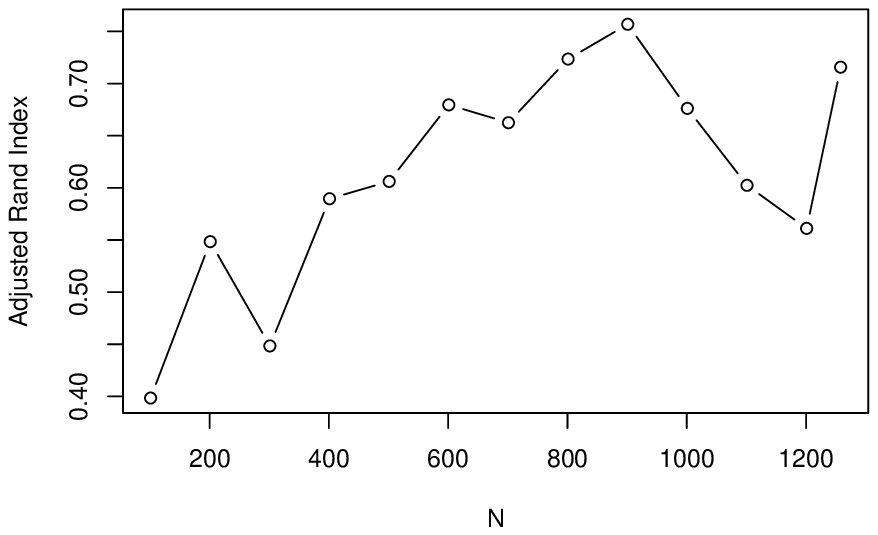}{6cm}{(b) ARI between HBCM and Spectral}
	\caption{ARI for clustering results for stock data. Left panel:  ARI of spectral clustering, WGCNA, and HBCM, each against the true labels. Right panel:  ARI between HBCM and Spectral.}
	\label{fig:stock}
\end{figure}

As shown in the left panel of Figure \ref{fig:stock},  WGCNA consistently yields low ARI scores due to imbalanced clustering results while HBCM outperforms spectral clustering in most cases in terms of ARI, particularly when $N$ is small. Additionally, the right panel reports the mutual ARI between HBCM and spectral clustering, indicating that HBCM can meaningfully deviate from the initial labels estimated by spectral clustering. 

\section{Additional results on mouse embryo single-cell RNA-seq data}
\begin{figure}[!ht]
\begin{center}
\includegraphics[width=3.5in]{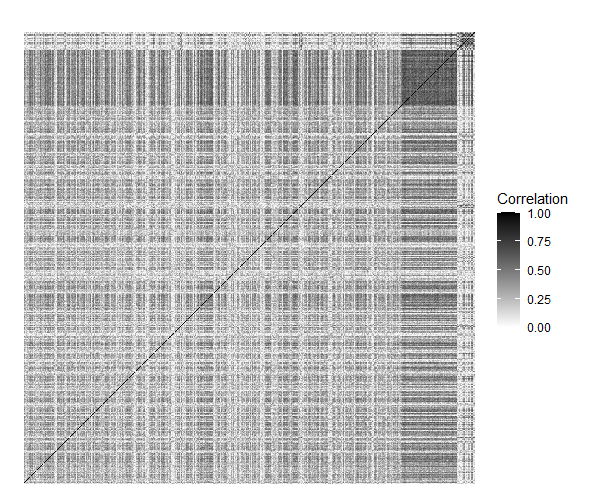}
\end{center}

\caption{Heatmaps of the absolute correlation matrix corresponding to WGCNA ($K=3$). }
\label{fig:mouse_wgcna}
\end{figure}

As in the simulation studies and stock price data analysis, WGCNA results in highly unbalanced cluster label assignments in the single-cell RNA-seq dataset. Figure \ref{fig:mouse_wgcna}
displays heatmaps of the absolute correlation matrix, with genes reordered according to the clusters identified by WGCNA. In the figure, cluster 1 dominates in size, and its within-cluster density is lower than the density between cluster 1 and 2, as further confirmed by the following matrix, which averages each block of the absolute correlation matrix.
\begin{align*}
B_{\textnormal{WGCNA}} = \begin{pmatrix}
0.24 & 0.34 & 0.19 \\
0.34 & 0.58 & 0.31 \\
0.19 & 0.31 & 0.47
\end{pmatrix}.
\end{align*}

Next, we perform gene set enrichment analysis (GSEA) to compare the identified clusters with gene sets associated with known biological functions. We use the MSigDB database (https://www.gsea-msigdb.org/gsea/msigdb) to screen for gene sets with significant overlap. The database only allows input of gene clusters with a maximum size of 500. Therefore, we increase the number of clusters to $K=8$ to ensure that all clusters identified by spectral clustering and WGCNA have a size no greater than 500. However, WGCNA consistently produces a single dominant gene cluster (size 1178 for $K=8$) and is therefore excluded from the study. We screen the clusters identified by spectral clustering and HBCM, reporting the predefined gene sets if the $p$-values are less than $10^{-20}$. The numbers of gene sets overlapping with each of the eight clusters identified by each method are reported in Table \ref{table:GSEA}, with clusters ranked by size. The full list of identified gene sets is provided in a separate Excel spreadsheet. Spectral clustering results overlap with 118 predefined gene sets from the MSigDB database, while HBCM results overlap with 116 gene sets. However, the gene sets overlapping with HBCM results are slightly more dispersed. For example, cluster 8 overlaps with 9 gene sets, which is substantial given the cluster's size.

\begin{table}[ht!]
\centering
\begin{adjustbox}{width=1.0\textwidth, center=\textwidth }
\begin{tabular}{|l|l|l|l|l|l|}
     \hline
\multicolumn{3}{|c|}{Spectral clustering} & \multicolumn{3}{c|}{HBCM}  \\
\hline 
Cluster index & Cluster size & \# Overlapping sets & Cluster index & Cluster size & \# Overlapping sets \\
\hline
1 &	209 	&1    & 1  & 403 &	23
 \\      
2 &	202 &	16   &  2  & 295 &	66
\\
3 &	200 &	1     & 3  & 219 &	2
\\
4 &	198 &	2     & 4  & 169 &	1
 \\
5 &	174 &	85    & 5  & 139 &	1
\\
6 &	172 &	1     & 6  & 98 &	12
\\
7 &	156 &	1     & 7 & 77 &	2
\\
8 &	112 &	11    & 8  & 23	&  9
 \\
\hline 
\end{tabular}
\end{adjustbox}
\caption{Gene set enrichment analysis (GSEA) of clustering results from spectral clustering and HBCM.} \label{table:GSEA}

\end{table}

\section{Derivation of the variational expectation-maximization algorithm}
In this section, we provide details on the factorization of  $q_1(\V{c}), q_2(\V{\alpha})$ and derivation of parameter estimators in closed-forms involved in the variational EM algorithm. Recall that the algorithm maximizes 
\begin{align*}
 J(q_1(\V{c}),q_2(\V{\alpha}), \Phi) = &\sum_{\V{c} \in [K]^P}\int_{\V{\alpha}\in \mathbb{R}^{N\times K}} q_1(\V{c})q_2( \V{\alpha})L(\M{X},\V{c},\V{\alpha}|\Phi)d\V{\alpha} \\
  & - \sum_{\V{c} \in [K]^P}\int_{\V{\alpha}\in \mathbb{R}^{N\times K}} q_1(\V{c})q_2( \V{\alpha})\log q_1(\V{c})q_2( \V{\alpha}) d\V{\alpha},
\end{align*}
over $q_1(\V{c})$, $q_2(\V{\alpha})$ and $\Phi$ iteratively, 
where 
\begin{align*}
 L(\M{X},\V{c},\V{\alpha}|\Phi)
= & \log \big[f(\V{c}|\V{\pi})f(\V{\alpha}|\bome)f(\M{X}|\V{c},\V{\alpha},\{\lambda_j\}, \{\sigma_j\})\big] \\
= & \sum_{j=1}^P\sum_{k=1}^K 1(c_j=k)\log{\pi_k} -\frac{N}{2}\log|\bome|-\frac{1}{2}\sum_{i=1}^N\V{\alpha}_i^\top\bome^{-1}\V{\alpha}_i \\
    &  +  \sum_{j=1}^P\sum_{i=1}^N\sum_{k=1}^K 1(c_j=k)\bigg( -\frac{1}{2}\log\sigma_j^2-\frac{(X_{ij}-\lambda_j\alpha_{ik})^2}{2\sigma_j^2}\bigg).
\end{align*}


\begin{proposition}
Let
\begin{align*}
    f_j(c_j) = \exp\left\{ \ln \pi_{c_j} +  \sum_{i=1}^N\bigg( -\frac{1}{2}\ln\sigma_j^2-\frac{\mathbb{E}_{q_2}(X_{ij}-\lambda_j\alpha_{ic_j})^2}{2\sigma_j^2}\bigg)\right\}.
\end{align*}
Given $\Phi$ and $q_2(\V{\alpha})$, 
\begin{align*}
\argmax_{q_1} J(q_1,q_2,\Phi)
\end{align*}
is achieved at 
\begin{align*}
 q_1(\V{c})=\prod_{j=1}^P\frac{f_j(c_j)}{\sum_{k=1}^K f_j(k)}.
\end{align*}
\end{proposition}

\begin{proof}
Given $q_2(\V{\alpha})$ and $\Phi$, updating $q_1(\V{c})$ is equivalent to maximizing
\begin{align*}
     J_1(q_1,q_2, \Phi)&=\sum_{\V{c} \in [K]^P}\ q_1(\V{c})\int_{\V{\alpha}\in \mathbb{R}^{N\times K}}q_2( \V{\alpha})L(\M{X},\V{c},\V{\alpha}|\Phi)d\V{\alpha}  - \sum_{\V{c} \in [K]^P} q_1(\V{c})\ln q_1(\V{c}) \\
     & = \sum_{\V{c} \in [K]^P}\ q_1(\V{c})\mathbb{E}_{q_2}(L(\M{X},\V{c},\V{\alpha}|\Phi)) - \sum_{\V{c} \in [K]^P} q_1(\V{c})\ln q_1(\V{c}).
\end{align*}
Up to a constant,
\begin{align*}
   \mathbb{E}_{q_2}(L(\M{X},\V{c},\V{\alpha}|\Phi)) & =\sum_{j=1}^P\sum_{k=1}^K 1(c_j=k)\ln{\pi_k}\\
   & \quad +  \sum_{j=1}^P\sum_{i=1}^N\sum_{k=1}^K 1(c_j=k)\bigg( -\frac{1}{2}\ln\sigma_j^2-\frac{\mathbb{E}_{q_2}(X_{ij}-\lambda_j\alpha_{ik})^2}{2\sigma_j^2}\bigg). 
\end{align*}
Next after rearranging the terms we obtain that
\begin{align*}
\mathbb{E}_{q_2}(L(\M{X},\V{c},\V{\alpha}|\Phi)) & = \sum_{j=1}^P\ln \pi_{c_j} + \sum_{j=1}^P\sum_{i=1}^N\bigg( -\frac{1}{2}\ln\sigma_j^2-\frac{\mathbb{E}_{q_2}(X_{ij}-\lambda_j\alpha_{ic_j})^2}{2\sigma_j^2}\bigg)\\
& =\sum_{j=1}^P\bigg[\ln \pi_{c_j} +  \sum_{i=1}^N\bigg( -\frac{1}{2}\ln\sigma_j^2-\frac{\mathbb{E}_{q_2}(X_{ij}-\lambda_j\alpha_{ic_j})^2}{2\sigma_j^2}\bigg)\bigg]\\
&= \ln \prod_{j=1}^P f_j(c_j).
\end{align*}
Then we have 
\begin{align*}
 J_1(q_1,q_2, \Phi)&= \sum_{\V{c} \in [K]^P}\ q_1(\V{c})\ln \prod_{j=1}^P f_j(c_j) - \sum_{\V{c} \in [K]^P} q_1(\V{c})\ln q_1(\V{c}). \\
 & = \sum_{\V{c} \in [K]^P} q_1(\V{c})\ln \frac{\prod_{j=1}^P f_j(c_j)}{q_1(\V{c})} \\
 & \leq \ln \sum_{\V{c} \in [K]^P}q_1(\V{c})\frac{\prod_{j=1}^P f_j(c_j)}{q_1(\V{c})} \quad \mbox(\text{Jensen's ineqaulity}) \\
 &=\ln \sum_{\V{c} \in [K]^P}\prod_{j=1}^P f_j(c_j).
\end{align*}
The equality holds when $\frac{\prod_{j=1}^P f_j(c_j)}{q_1(\V{c})}$ is constant with respect to  $\V{c}$. Therefore,
\begin{align*}
 q_1(\V{c})=\frac{\prod_{j=1}^P f_j(c_j)}{\sum_{c_1'=1}^K \sum_{c_2'=1}^K \dots \sum_{c_P'=1}^K\prod_{j=1}^P f_j(c_j') }=\prod_{j=1}^P\frac{f_j(c_j)}{\sum_{k=1}^Kf_j(k)}.
\end{align*}
\end{proof}


\begin{proposition}
Let
\begin{align*}
    g_i(\V{\alpha}_i) =\exp\left\{-\frac{1}{2} \V{\alpha}_i^\top\bome^{-1}\V{\alpha}_i-\frac{1}{2}\sum_{j=1}^P\sum_{k=1}^K\frac{\mathbb{E}_{q_1}[1(c_j=k)]}{\sigma_j^2}(X_{ij}-\lambda_j\alpha_{ik})^2\right\}.
\end{align*}
Given $\Phi$ and $q_1(\V{c})$, 
\begin{align*}
\argmax_{q_1} J(q_1,q_2,\Phi)
\end{align*}
is achieved at 
\begin{align*}
q_2(\V{\alpha})=\prod_{i=1}^N\frac{g_i(\alpha_i)}{\int_{\alpha_i'\in \mathbb{R}^K}g_i(\alpha_i')d\alpha_i'}.
\end{align*}
Furthermore, let  
\begin{align*}
\M{D}_{j}&= \frac{\lambda_j^2}{\sigma_j^2}\begin{pmatrix}
\mathbb{E}_{q_1}(1(c_j=1)) & 0 & \cdots & 0 \\
0 & \mathbb{E}_{q_1}(1(c_j=2)) & \cdots & 0 \\
\vdots & \vdots & \ddots & \vdots\\
0 & 0 & \cdots & \mathbb{E}_{q_1}(1(c_j=K))
\end{pmatrix}_{K \times K},\\
    \V{b}_{ij} &=\frac{1}{\lambda_j}\begin{pmatrix}
    X_{ij}\\
    X_{ij}\\
    \vdots \\
    X_{ij}
    \end{pmatrix}_{K\times 1},\\
    \M{A} & =\bome^{-1}+\sum_{j=1}^P\M{D}_j, \,\, \V{B}_i  =\sum_{j=1}^P\M{D}_j\V{b}_{ij}, \,\, \hat{\M{V}}  =\M{A}^{-1}, \\
    \textnormal{and } \hat{\V{\mu}}_i & =\M{A}^{-1}\V{B}_i.
\end{align*}
Then $q_2(\alpha_i)$ is the density function of a $K$-variate normal distribution $N(\hat{\V{\mu}}_i,\hat{\M{V}})$.
\end{proposition}

\begin{proof}
Given $q_1(\V{c})$ and $\Phi$, updating $q_2(\V{\alpha})$ is equivalent to maximizing
\begin{align*}
J_2(q_1,q_2, \Phi)&=\sum_{\V{c} \in [K]^P}\int_{\V{\alpha}\in \mathbb{R}^{N\times K}} q_1(\V{c})q_2( \V{\alpha})L(\M{X},\V{c},\V{\alpha}|\Phi)d\V{\alpha} - \int_{\V{\alpha}\in \mathbb{R}^{N\times K}}q_2( \V{\alpha})\ln q_2( \V{\alpha})d\V{\alpha}\\
& = \int_{\V{\alpha}\in \mathbb{R}^{N\times K}} q_2( \V{\alpha})\mathbb{E}_{q_1}(L(\M{X},\V{c},\V{\alpha}|\Phi))d\V{\alpha}- \int_{\V{\alpha}\in \mathbb{R}^{N\times K}}q_2( \V{\alpha})\ln q_2( \V{\alpha})d\V{\alpha},
\end{align*}
and up to a constant we have
\begin{align*}
\mathbb{E}_{q_1}(L(\M{X},\V{c},\V{\alpha}|\Phi)) & =  -\frac{1}{2}\sum_{i=1}^N\V{\alpha}_i^\top\bome^{-1}\V{\alpha}_i   +  \sum_{j=1}^P\sum_{i=1}^N\sum_{k=1}^K \mathbb{E}_{q_1}[1(c_j=k)]\bigg(-\frac{(X_{ij}-\lambda_j\alpha_{ik})^2}{2\sigma_j^2}\bigg).
\end{align*}

Next after rearranging of the terms we obtain that
\begin{align*}
 \mathbb{E}_{q_1}(L(\M{X},\V{c},\V{\alpha}|\Phi)) & =  \sum_{i=1}^N\bigg[-\frac{1}{2} \V{\alpha}_i^\top\bome^{-1}\V{\alpha}_i-\frac{1}{2}\sum_{j=1}^P\sum_{k=1}^K\frac{\mathbb{E}_{q_1}[1(c_j=k)]}{\sigma_j^2}(X_{ij}-\lambda_j\alpha_{ik})^2\bigg]\\
 & = \ln\prod_{i=1}^Ng_i(\V{\alpha}_i).
\end{align*}

Then we have
\begin{align*}
    J_2(q_1,q_2, \Phi)&=\int_{\V{\alpha}\in \mathbb{R}^{N\times K}} q_2( \V{\alpha})\ln\prod_{i=1}^Ng_i(\V{\alpha}_i)d\V{\alpha}- \int_{\V{\alpha}\in \mathbb{R}^{N\times K}}q_2( \V{\alpha})\ln q_2( \V{\alpha})d\V{\alpha}\\
    & = \int_{\V{\alpha}\in \mathbb{R}^{N\times K}}q_2( \V{\alpha})\ln \frac{\prod_{i=1}^Ng_i(\V{\alpha}_i)}{q_2( \V{\alpha})}d\V{\alpha}\\
    & \leq \ln\int_{\V{\alpha}\in \mathbb{R}^{N\times K}}\prod_{i=1}^Ng_i(\V{\alpha}_i)d\V{\alpha}.
\end{align*}
The equality holds when $\frac{\prod_{i=1}^Ng_i(\V{\alpha}_i)}{q_2(\V{\alpha})}$ is constant with respect to $\V{\alpha}$, i.e.
\begin{align*}
    q_2(\V{\alpha}) = \frac{\prod_{i=1}^Ng_i(\V{\alpha}_i)}{\int_{\V{\alpha}^{'}\in \mathbb{R}^{N\times K}}\prod_{i=1}^Ng_i(\V{\alpha}_i^{'})d\V{\alpha}'} =\prod_{i=1}^N\frac{g_i(\V{\alpha}_i)}{\int_{\V{\alpha}_i\in \mathbb{R}^K}g_i(\V{\alpha}_i^{'})d\V{\alpha}_i^{'}}.
\end{align*}

Lastly we show that  $q_2(\V{\alpha}_i)$ is density function of a $K$-variate normal distribution $N(\hat{\V{\mu}}_i, \hat{\M{V}})$. We  reorganize the terms of $q_2(\V{\alpha}_i)$ as a quadratic function of $\V{\alpha}_i$,
\begin{align*}
Q(\V{\alpha}_i)&= \V{\alpha}_i^\top\bome^{-1}\V{\alpha}_i+\sum_{j=1}^P\sum_{k=1}^K\frac{\mathbb{E}_{q_1}[1(c_j=k)]}{\sigma_j^2}(X_{ij}-\lambda_j\alpha_{ik})^2\\
& = \V{\alpha}_i^\top\bome^{-1}\V{\alpha}_i+\sum_{j=1}^P(\V{\alpha}_i-\V{b}_{ij})^\top \M{D}_j (\V{\alpha}_i-\V{b}_{ij})\\
& \propto \V{\alpha}_i^\top\bome^{-1}\V{\alpha}_i +\sum_{j=1}^P\V{\alpha}_i^\top \M{D}_j\V{\alpha}_i - \left(\sum_{j=1}^P\V{b}_{ij}^\top\M{D}_j\right)\V{\alpha}_i - \V{\alpha}_i^\top\left(\sum_{j=1}^P\M{D}_j\V{b}_{ij}\right)\\
& = \V{\alpha}_i^\top\left(\bome^{-1}+\sum_{j=1}^P\M{D}_j\right)\V{\alpha}_i - \left(\sum_{j=1}^P\V{b}_{ij}^\top\M{D}_j\right)\V{\alpha}_i - \V{\alpha}_i^\top\left(\sum_{j=1}^P\M{D}_j\V{b}_{ij}\right) \\
&\propto  \V{\alpha}_i^\top \M{A}\V{\alpha}_i - \V{B}_i^\top\V{\alpha}_i -\V{\alpha}_i^\top\V{B}_i + \V{B}_i^\top\M{A}^{-1}\V{B}_i\\ & = \V{\alpha}_i^\top \M{A}\V{\alpha}_i - \V{B}_i^\top\M{A}^{-1}\M{A}\V{\alpha}_i -\V{\alpha}_i^\top\M{A}\M{A}^{-1}\V{B}_i + \V{B}_i^\top\M{A}^{-1}\M{A}\M{A}^{-1}\V{B}_i \\
    &\propto (\V{\alpha}_i-\hat{\V{\mu}}_i)^\top\hat{\M{V}}^{-1}(\V{\alpha}_i-\hat{\V{\mu}}_i).
\end{align*}
Therefore $q_2(\V{\alpha}_i)$ follows a $K$-variate normal distribution $N(\hat{\V{\mu}}_i, \hat{\M{V}})$.
\end{proof}
We now turn to the M-step. 
\begin{proposition}
Let
\begin{align*}
J_3(q_1,q_2, \Phi)&=\sum_{\V{c} \in [K]^P}\int_{\V{\alpha}\in \mathbb{R}^{N\times K}} q_1(\V{c})q_2( \V{\alpha})L(\M{X},\V{c},\V{\alpha}|\Phi)d\V{\alpha}
\end{align*}
Then the maximizer of $J_3(q_1,q_2, \Phi)$ over $\Phi$ given $q_1$ and $q_2$ is achieved at
\begin{align*} 
\hat{\M{\Omega}} & = \frac{1}{N}\sum_{i=1}^N\mathbb{E}_{q_2}(\V{\alpha}_i\V{\alpha}_i^\top), \\
             \hat{\pi}_k & = \frac{\sum_{j=1}^P \mathbb{E}_{q_1}[1(c_j=k)]}{\sum_{j=1}^P\sum_{k'=1}^K \mathbb{E}_{q_1}[1(c_j=k')]}, \quad k=1,...,K,
            \\
             \hat{\lambda}_j & = \frac{\sum_{i=1}^N\sum_{k=1}^K \mathbb{E}_{q_1}[1(c_j=k)]X_{ij}\mathbb{E}_{q_2}[\alpha_{ik}]}{\sum_{i=1}^N\sum_{k=1}^K \mathbb{E}_{q_1}[1(c_j=k)]\mathbb{E}_{q_2}[\alpha_{ik}^2]}, \quad  j= 1,...,P,\\
             \hat{\sigma}_j^2 & = \frac{\sum_{i=1}^N\sum_{k=1}^K \mathbb{E}_{q_1}[1(c_j=k)]
    (X_{ij}^2+\hat{\lambda}_j^2\mathbb{E}_{q_2}[\alpha_{ik}^2]-2\hat{\lambda}_jX_{ij}\mathbb{E}_{q_2}[\alpha_{ik}])}{N},\quad  j= 1,...,P.
\end{align*}
\end{proposition}
\begin{proof}
Note that
\begin{align*}
    J_3(q_1,q_2, \Phi)&=\sum_{\V{c} \in [K]^P}\int_{\V{\alpha}\in \mathbb{R}^{N\times K}} q_1(\V{c})q_2( \V{\alpha})L(\M{X},\V{c},\V{\alpha}|\Phi)d\V{\alpha}\\
    & = \mathbb{E}_{q_1 \times q_2}(L(\M{X},\V{c},\V{\alpha}|\Phi)) \\
 & = \sum_{j=1}^P\sum_{k=1}^K \mathbb{E}_{q_1}[1(c_j=k)]\log{\pi_k} -\frac{N}{2}\log|\M{\Omega}|-\frac{1}{2}\sum_{i=1}^N\V{\alpha}_i^\top\M{\Omega}^{-1}\V{\alpha}_i \\
    & \quad +  \sum_{j=1}^P\sum_{i=1}^N\sum_{k=1}^K \mathbb{E}_{q_1}[1(c_j=k)]\bigg( -\frac{1}{2}\log\sigma_j^2-\frac{X_{ij}^2-2X_{ij}\lambda_j\mathbb{E}_{q_2}[\alpha_{ik}]+\lambda_j^2\mathbb{E}_{q_2}[\alpha_{ik}^2]}{2\sigma_j^2}\bigg).
\end{align*}
The result then follows by straightforward calculation.
\end{proof}

\section{Proofs in Section 3.1}
We start with a lemma on formula (6) and positivity of $\bome$.
\begin{lemma}\label{lemmaS1}
The covariance defined in model (1) can be written as formula (6). Moreover, It is necessary to require $\M{\Omega}$ to be positive semi-definite to make $\M{\Sigma}$ a well-defined positive definite covariance matrix for arbitrary $\V{\sigma}^2$.
\end{lemma}
\begin{proof}[Proof of Lemma \ref{lemmaS1}]
Recall formula \eqref{e6} gives $\bSig=\diag(\blam)\M{L}\bome\M{L}^{\top}\diag(\blam)+\diag(\bsig^2)$.
It is straightforward to check that the $(j,j')$-entry of $\bSig$ in formula (6) is equivalent to the equations given in model (1). 

Now we show the necessity of positive semi-definiteness of $\bome$ by contradiction. If $\bome$ has a negative eigenvalue $\zeta$ and $\M{v}$ is the corresponding eigenvector, then we have $\bome\M{v}=\zeta\M{v}$. Because $\M{L}^{\top}\diag(\blam)$ is a full rank matrix that represents a surjection from $\mathbb{R}^P$ to $\mathbb{R}^K$, there exists $\M{v}'\in\mathbb{R}^P$ such that $\M{v}=\M{L}^{\top}\diag(\blam)\M{v}'$. Therefore, we have
\begin{align*}
\M{v}'^{\top}\bSig\M{v}'=&\M{v}'^{\top}\diag(\blam)\M{L}\bome\M{L}^{\top}\diag(\blam)\M{v}'+\M{v}'^{\top}\diag(\bsig^2)\M{v}'\\
=&\M{v}^{\top}\bome\M{v}+\M{v}'^{\top}\diag(\bsig^2)\M{v}'\\
=&\zeta +\M{v}'^{\top}\diag(\bsig^2)\M{v}'.
\end{align*}
As $\zeta<0$, any tiny positive scalar matrix like $\diag(\bsig^2)=\zeta'\M{I}$ with $\zeta'<|\zeta|/\|\M{v}'\|^2$ will lead to a negative value of $\M{v}'^{\top}\bSig\M{v}'$. Since we require $\bSig$ to be positive definite covariance matrix, $\bome$ has to be positive semi-definite. 
\end{proof}
\begin{proof}[Proof of Theorem 1]
For the ``if'' part, it suffices to show
\[\diag(\blam)\M{L}\bome\M{L}^{\top}\diag(\blam)=\diag(\tilde{\blam})\tilde{\M{L}}\tilde{\bome}\tilde{\M{L}}^{\top}\diag(\tilde{\blam}).\]
Direct calculation shows
\begin{align*}
    &\diag(\tilde{\blam})\tilde{\M{L}}\tilde{\bome}\tilde{\M{L}}^{\top}\diag(\tilde{\blam})\\
   =&\diag(\tilde{\blam}) \M{L}\M{P}\M{P}^{\top}\diag({\V{d}})\bome\diag({\V{d}})\M{P} \M{P}^{\top}\M{L}^{\top}\diag(\tilde{\blam})\\     
   =&\diag(\tilde{\blam}) \M{L}\diag({\V{d}})\bome\diag({\V{d}})\M{L}^{\top}\diag(\tilde{\blam}).\\
\end{align*}
Therefore, the Lemma follows the fact 
\[\diag(\tilde{\blam}) \M{L}\diag({\V{d}})=\diag(\blam\circ (\M{L}\V{d}^{-1})) \M{L}\diag({\V{d}}) = \diag(\blam)\diag(\M{L}\V{d}^{-1})) \M{L}\diag({\V{d}})=\diag(\blam)\M{L}.\]

Now we show the ``only if'' part. If two systems $\{\blam,\bome, \M{L}, \bsig^2\}$ and $\{\tilde\blam,\tilde\bome,\tilde{\M{L}},\tilde{\bsig}^2\}$ satisfy condition 1 and define the same covariance, then 
\[\diag(\blam)\M{L}\bome\M{L}^{\top}\diag(\blam)+\diag(\bsig^2)=\diag(\tilde{\blam})\tilde{\M{L}}\tilde{\bome}\tilde{\M{L}}^{\top}\diag(\tilde{\blam})+\diag(\tilde\bsig^2).\] It leads to
\[\diag(\blam)\M{L}\bome\M{L}^{\top}\diag(\blam)-\diag(\tilde{\blam})\tilde{\M{L}}\tilde{\bome}\tilde{\M{L}}^{\top}\diag(\tilde{\blam})=\diag(\tilde\bsig^2)-\diag(\bsig^2),\]
which implies
\begin{equation}\label{temp6}
\M{L}\bome\M{L}^{\top} -\diag(\blam^{-1}\circ\tilde{\blam})\tilde{\M{L}}\tilde{\bome}\tilde{\M{L}}^{\top}\diag(\tilde{\blam}\circ\blam^{-1})=\diag(\blam^{-1})(\diag(\tilde\bsig^2)-\diag(\bsig^2))\diag(\blam^{-1}).    
\end{equation}


Let $\{\M{e}_j=(0,...,0,1,0,...,0)^{\top}\}_{j=1}^P$ be the standard coordinate basis of $\mathbb{R}^P$. If $j$ and $j'$ belong to the same group in the first system, i.e., $c_j=c_{j'}$, then $L_{jk}=L_{j'k}$ for all $k\in[K]$, which implies $\M{L}^{\top}\M{e}_j=\M{L}^{\top}\M{e}_{j'}$. Therefore, $\M{L}\bome\M{L}^{\top}(\M{e}_j-\M{e}_{j'})=0$.

Now pick $j$ and $j'$ with $c_j=c_{j'}$. Multiplying $\M{e}_j-\M{e}_{j'}$ on both sides of \eqref{temp6} gives
\[-\diag(\blam^{-1}\circ\tilde{\blam})\tilde{\M{L}}\tilde{\bome}\tilde{\M{L}}^{\top}\diag(\tilde{\blam}\circ\blam^{-1})(\M{e}_j-\M{e}_{j'})=\frac{\tilde{\sigma}^2_j-\sigma^2_j}{\lambda_j^2}\M{e}_j-\frac{\tilde{\sigma}^2_{j'}-\sigma^2_{j'}}{\lambda_{j'}^2}\M{e}_{j'},\]
which implies
\begin{equation}
\tilde{\M{L}}\tilde{\bome}\tilde{\M{L}}^{\top}\diag(\tilde{\blam}\circ\blam^{-1})(\M{e}_j-\M{e}_{j'})=-\frac{\tilde{\sigma}^2_j-\sigma^2_j}{\lambda_j\tilde{\lambda}_{j}}\M{e}_j+\frac{\tilde{\sigma}^2_{j'}-\sigma^2_{j'}}{\lambda_{j'}\tilde{\lambda}_{j'}}\M{e}_{j'}.    \label{temp7}
\end{equation}

Now we show the $K$-dimensional vector $\M{w}=\tilde{\bome}\tilde{\M{L}}^{\top}\diag(\tilde{\blam}\circ\blam^{-1})(\M{e}_j-\M{e}_{j'})$ must be $\M{0}$. If not, say, $w_k\ne0$, then $(\tilde{\M{L}}\M{w})_t=w_k\ne0$ for all $t$ with $\tilde L_{tk}=1$. By Condition 1, there are at least 3 features in each community, which means there are at least 3 nonzero entries in vector $\tilde{\M{L}}\M{w}$. This contradicts with \eqref{temp7}, where the right hand side contains at most 2 nonzero entries. Thus, $\M{w}=\M{0}$ and $\tilde{\sigma}^2_j=\sigma^2_j$.

Since $\tilde{\bome}>0$, $\M{w}=\M{0}$ implies $\tilde{\M{L}}^{\top}\diag(\tilde{\blam}\circ\blam^{-1})(\M{e}_j-\M{e}_{j'})=0$. Therefore, $\tilde{c}_j=\tilde{c}_{j'}$, and $\tilde{\lambda}_j\lambda_j^{-1}=\tilde{\lambda}_{j'}\lambda_{j'}^{-1}$. 

In summary, we have shown that (1) $c_j=c_{j'}$ implies $\tilde{c}_j=\tilde{c}_{j'}$; (2) $\tilde{\sigma}^2_j=\sigma^2_j$; (3) $\tilde{\lambda}_j\lambda_j^{-1}=\tilde{\lambda}_{j'}\lambda_{j'}^{-1}$ whenever $c_j=c_{j'}$. In particular, (1) indicates that any two parameter systems share the same number of communities and the same membership (up to a permutation) for all the features. By (3), we can define a vector $\V{d}$ with $d_k=\lambda_j\tilde{\lambda}_j^{-1}$ when $c_j=k$.  

\end{proof}

\begin{proof}[Proof of Proposition 1]
By Theorem 1 and our convention ignoring the label permutation, the two systems $\{\blam,\bome, \M{L}, \bsig^2\}$ and $\{\tilde\blam,\tilde\bome,\tilde{\M{L}},\tilde{\bsig}^2\}$ satisfy $\tilde\bome\ =\diag({\V{d}'})\bome\diag({\V{d}'})$, $\tilde{\blam} =\blam\circ (\M{L}\V{d}'^{-1})$, $\M{L}=\tilde{\M{L}}$, $\bsig^2=\tilde\bsig^2$ for some $\V{d}'$. 

By definition of the canonical parameters, the canonical system obtained from $\{\blam,\bome, \M{L}, \bsig^2\}$, denoted by $\{\blam^*,\bome^*, \M{L}, \bsig^2\}$, is up to a multiplier vector $\V{d}=(\bome\M{L}^{\top}\diag(\blam)\M{1}_P/P)^{-1}$. That is, 
  $\blam^*= \blam\circ(\M{L}\V{d}^{-1})$, $\bome^*=\diag({\V{d}})\bome\diag({\V{d}})$. 
  
Similarly,  the canonical system $\{\tilde\blam^*,\tilde\bome^*, \tilde{\M{L}}, \tilde\bsig^2\}$ obtained from $\{\tilde\blam,\tilde\bome, \tilde{\M{L}}, \tilde\bsig^2\}$ satisfies $\tilde\blam^*= \tilde\blam\circ(\tilde{\M{L}}\tilde{\V{d}}^{-1})$, $\tilde\bome^*=\diag({\tilde{\V{d}}})\tilde\bome\diag(\tilde{\V{d}})$, where  $\tilde{\V{d}}=(\tilde\bome\tilde{\M{L}}^{\top}\diag(\tilde\blam)\M{1}_P/P)^{-1}$. 

Now we show $\tilde\blam^*=\blam^*$ and $\tilde\bome^*=\bome^*$. It suffices to show $\V{d}=\V{d}'\circ\tilde{\V{d}}$. Direct calculation shows
\begin{align*}
   \tilde{\V{d}}^{-1} &=\tilde\bome\tilde{\M{L}}^{\top}\diag(\tilde\blam)\M{1}_P/P\\
   &=\diag({\V{d}'})\bome\diag({\V{d}'}) \M{L}^{\top}\diag(\blam\circ (\M{L}\V{d}'^{-1}))\M{1}_P/P\\
   &=\V{d}'\circ\bome\diag({\V{d}'}) \M{L}^{\top}\diag(\M{L}\V{d}'^{-1}))\diag(\blam) \M{1}_P/P\\
    &=\V{d}'\circ\bome\M{L}^{\top}\diag(\blam)\M{1}_P/P\\
     &=\V{d}'\circ\V{d}^{-1}.
\end{align*}
That is, $\V{d}=\V{d}'\circ\tilde{\V{d}}$. 
\end{proof}

\begin{proof}[Proof of Proposition 2]
The conclusion can be verified directly.
\end{proof}

\section{Proofs in Section 3.2}
\begin{proof}[Proof of Proposition 3]
Without loss of generality, we assume $0<\epsilon<1$. Given the model assumption that $\text{Cov}(X_{j}, X_{j'}|\V{c}) = \lambda_j\lambda_{j'}\omega_{c_j c_{j'}} +  \sigma_j^2 1(j=j')$, we have
\begin{align*}
&   \mathbb{P}(|\hat{\lambda}_j - \lambda_j^*|\geq \epsilon| \V{c}^*)\\
= \,\, & \mathbb{P}(|\hat{\lambda}_j - \mathbb{E}(\hat{\lambda}_j) + \mathbb{E}(\hat{\lambda}_j) - \lambda_j^*|\geq \epsilon| \V{c}^*)\\
\leq \,\, & \mathbb{P}(|\hat{\lambda}_j - \mathbb{E}(\hat{\lambda}_j)|\geq \epsilon/2| \V{c}^*) + \mathbb{P}(| \mathbb{E}(\hat{\lambda}_j) - \lambda_j^*|\geq \epsilon/2|\V{c}^*)\\
= \,\, & \mathbb{P}\left(\left|\sum_{j'=1}^P\frac{\sum_{i=1}^NX_{ij}X_{ij'}}{PN} - \sum_{j'=1}^P \frac{\lambda_j\lambda_{j'}\omega_{c_jc_{j'}}}{P} - \frac{\sigma_j^{*2}}{P}\right|\geq \epsilon/2 \middle| \V{c}^*\right) + \mathbb{P}\left(\left| \frac{\sigma_j^{*2}}{P}\right|\geq \epsilon/2 \middle | \V{c}^*\right)\\
\leq \,\, & \sum_{j'=1}^P\mathbb{P}\left(\left|\frac{\sum_{i=1}^NX_{ij}X_{ij'}}{N} - \lambda_j\lambda_{j'}\omega_{c_jc_{j'}} - \sigma_j^{*2}1(j=j')\right|\geq \epsilon/2 \middle | \V{c}^*\right) + o(1) \\ 
\leq \,\, & PC_1\exp(-C_2N\epsilon^2) + o(1),
\end{align*}
where $C_1$ and $C_2$ are constants. The last inequality is due to the tail bound of the product of two sub-gaussian variables being sub-exponential, and can be obtained by Lemma B.4 in \cite{hao2014interaction}. Therefore the conclusion holds as $\log(P)/N = o(1)$.

\end{proof}

\begin{proof}[Proof of Proposition 4]
The difference between $\bar{J}_{\textnormal{core}}(\V{I}^{c}, \V{\alpha}^* )$ and $\bar{J}_{\textnormal{core}}(\bqc, \V{\mu} )$ is:
\begin{align*}
&  \bar{J}_{\textnormal{core}}(\V{I}^{c}, \V{\alpha}^* )-  \bar{J}_{\textnormal{core}}(\bqc, \V{\mu} ) \\
& =  \sum_{i=1}^N\sum_{j=1}^P\sum_{k=1}^K\left(-\frac{1}{2}\lambda_j^{*2}I_{jk}^c \alpha_{ik}^{*2}\right)  + \sum_{i=1}^N\sum_{j=1}^P\sum_{k=1}^K I_{jk}^c\lambda_j^{*}\lambda_j^{*}\alpha^*_{ic_j^*}\alpha^*_{ik}  \\
& \quad - \sum_{i=1}^N\sum_{j=1}^P\sum_{k=1}^K\left(-\frac{1}{2}\lambda_j^{*2}q_{jk}^c \mu_{ik}^2\right)  - \sum_{i=1}^N\sum_{j=1}^P\sum_{k=1}^K q_{jk}^c\lambda_j^{*}\lambda_j^{*}\alpha^*_{ic_j^*}\mu_{ik} \\
& = \sum_{i=1}^N\sum_{j=1}^P \frac{1}{2} \lambda_j^{*2} \left [-\sum_{k=1}^K I_{jk}^c \alpha_{ik}^{*2} +2\sum_{k=1}^K I_{jk}^c\alpha^*_{ic_j^*}\alpha^*_{ik}+\sum_{k=1}^K q_{jk}^c \mu_{ik}^2 -2\sum_{k=1}^K q_{jk}^c\alpha^*_{ic_j^*}\mu_{ik} \right ] \\
& = \sum_{i=1}^N\sum_{j=1}^P \frac{1}{2} \lambda_j^{*2} \left [-\sum_{k=1}^K I_{jk}^c \alpha_{ik}^{*2} +2\sum_{k=1}^K I_{jk}^c\left (\sum_{l=1}^K I_{jl}^c \alpha_{il}^* \right )\alpha^*_{ik}-2\sum_{k=1}^K q_{jk}^c \left (\sum_{l=1}^K I_{jl}^c \alpha_{il}^* \right )\mu_{ik}+\left ( \sum_{k=1}^K q_{jk}^c \mu_{ik} \right )^2 \right.  \\
& \quad \quad \quad \quad \quad \quad \quad \left . +\sum_{k=1}^K q_{jk}^c \mu_{ik}^2 -\left ( \sum_{k=1}^K q_{jk}^c \mu_{ik} \right )^2 \right ]  \\
& = \sum_{i=1}^N\sum_{j=1}^P \frac{1}{2} \lambda_j^{*2} \left [\sum_{k=1}^K I_{jk}^c\left (\sum_{l=1}^K I_{jl}^c \alpha_{il}^* \right )\alpha^*_{ik}-2\sum_{k=1}^K q_{jk}^c \left (\sum_{l=1}^K I_{jl}^c \alpha_{il}^* \right )\mu_{ik}+\left ( \sum_{k=1}^K q_{jk}^c \mu_{ik} \right )^2 \right.  \\
& \quad \quad \quad \quad \quad \quad \quad \left . +\sum_{k=1}^K q_{jk}^c \mu_{ik}^2 -\left ( \sum_{k=1}^K q_{jk}^c \mu_{ik} \right )^2 \right ] \\
& =\sum_{i=1}^N\sum_{j=1}^P \frac{1}{2} \lambda_j^{*2} \left [ \left ( \sum_{k=1}^K I_{jk}^c \alpha_{ik}^*-\sum_{k=1}^K q_{jk}^c \mu_{ik} \right )^2+\sum_{k=1}^K q_{jk}^c \mu_{ik}^2 -\left ( \sum_{k=1}^K q_{jk}^c \mu_{ik} \right )^2 \right ] \tag{7} \label{eq.s1}
\end{align*}
Equation (\ref{eq.s1}) is equal to zero if and only if the followings conditions hold:
\begin{itemize}
    \item[(C-1)] $\sum_{k=1}^KI_{jk}^c\alpha_{ik}^* = \sum_{k=1}^Kq_{jk}^c\mu_{ik}$, for $1\leq i \leq N, 1 \leq j \leq P$.
    \item[(C-2)]
    $(\sum_{k=1}^Kq_{jk}^c\mu_{ik}^2) - (\sum_{k=1}^Kq_{jk}^c\mu_{ik})^2 = 0$, for  $1\leq i \leq N, 1 \leq j \leq P$.
\end{itemize}
By Jensen's inequality, the equality in (C-2) holds only if either row $q_{j\cdot}^c$ has only $\{0,1\}$ as entries with row summation equal to one, or $\mu_{ik}$'s have constant values across $k$ at which $q_{jk}^c$ are non-zero for all $i$.

Recall that $\bic, \balp^*, \bqc, \V{\mu}$ denote  $[I_{jk}^c]_{P\times K}$, $[\alpha_{ik}^*]_{N\times K}$, $[q_{jk}^c]_{P\times K}$, and $[\mu_{ik}]_{N\times K}$, respectively. Then (C-1) can be written as $\bic\balp^{*T} = \bqc \V{\mu}^\top$, with condition that $\bic\balp^{*T}$ is of full rank $K$, i.e. $\bic, \balp^*$ are of full rank $K$ respectively, which guarantees that three is no missing class and no collinearity among columns of $\balp^*$. Consequently, $\bqc\V{\mu}^\top$ is also of full rank $K$ with each component matrix being full rank $K$.

Observing the fact that $\V{\mu}$ is of full rank $K$, we use the proof of contradiction to show that $\bqc$ is a $P\times K$ matrix with only $\{0,1\}$ entries, with each row summing to one.

Suppose for the row $j$ of $\bqc$, there are two non-zero entries, i.e. $q_{jk}^c\neq 0, q_{jk'}^c\neq 0\,\, (k\neq k')$. Then by condition (C-2),  $\mu_{ik} = \mu_{ik'}$ for all $i$,  which implies that there are duplicate columns in $\V{\mu}$, contradicting to the statement  that $\V{\mu}$ is of full rank $K$. Therefore $\bqc$ is a $P\times K$ matrix with only $\{0,1\}$ entries and each row sums to one. It is straightforward to show that $\bqc = \bic $ and $\V{\mu} = \balp^*$ is the solution to equation (\ref{eq.s1}) equal to zero. 

Finally, we prove that the solution is unique up to the permutation of labels. 
Let $\alpha_{\cdot k}^*, \mu_{\cdot k}$ denote the column $k$ of the corresponding matrix. Then the $j$-th row of matrix $\bic \balp^{*T}$ is
\begin{align*}
    [\bic \balp^{*T}]_{j\cdot} & = I^c_{j1}(\alpha_{\cdot 1}^*)^\top + \cdots+I^c_{jK}(\alpha_{\cdot K}^*)^\top
\end{align*}
together with the condition that $1 = \sum_{k=1}^KI_{jk}^c$, it
implies that the $j$-th row of matrix $\bic \balp^{*T}$ is the $k$-th column of $\balp^*$ where $I_{jk}^c = 1$. Therefore, the rows of $\bic \balp^{*T}$ is made of the columns of $\balp^*$. The number of times that the $k$-th column of $\balp^*$ appearing in the rows of $\bic \balp^{*T}$ depends on the number of objects that belong to the group $k$. Similarly, since $\bqc$ only has $\{0,1\}$ entries and each row sums to one, the rows of $\bqc\V{\mu}^\top$ consist of columns of $\V{\mu}$. Consequently, condition  $\bic\balp^\top = \bqc\V{\mu}^\top$ holds if and only if that after a column permutation $\V{P}_{K\times K}$, $\V{\mu}\V{P} = \balp^*$ and  $\bqc\V{P} = \bic$. Therefore, $\bqc = \bic $ and $\V{\mu} = \balp^*$ is the unique solution up to the permutation of labels. 
\end{proof}

\begin{proof}[Proof of Proposition 5]
Let $C_l=\{j: c_j^*=l \}$.
From equation \eqref{eq.s1},
\begin{align*}
&  \bar{J}_{\textnormal{core}}(\V{I}^{c}, \V{\alpha}^* )-  \bar{J}_{\textnormal{core}}(\bqc, \V{\mu} ) \\
& =\sum_{i=1}^N\sum_{j=1}^P \frac{1}{2} \lambda_j^{*2} \left [ \left ( \sum_{k=1}^K I_{jk}^c \alpha_{ik}^*-\sum_{k=1}^K q_{jk}^c \mu_{ik} \right )^2+\sum_{k=1}^K q_{jk}^c \mu_{ik}^2 -\left ( \sum_{k=1}^K q_{jk}^c \mu_{ik} \right )^2 \right ] \\
& \geq \min_j\left\{\frac{1}{2}\lambda_j^{*^2}\right\}\sum_{i=1}^N\sum_{j=1}^P \left [ \left ( \sum_{k=1}^K I_{jk}^c \alpha_{ik}^*-\sum_{k=1}^K q_{jk}^c \mu_{ik} \right )^2+\sum_{k=1}^K q_{jk}^c \mu_{ik}^2 -\left ( \sum_{k=1}^K q_{jk}^c \mu_{ik} \right )^2 \right ] \\
& = \frac{1}{2} \gamma_1^2 \sum_{i=1}^N\sum_{l=1}^K\sum_{j\in C_l}\bigg[\left(\alpha_{il}^* - \sum_{k=1}^Kq_{jk}^c \mu_{ik}\right)^2 +\left(\sum_{k=1}^Kq_{jk}^c\mu_{ik}^2\right) - \left(\sum_{k=1}^Kq_{jk}^c\mu_{ik}\right)^2  \bigg]\\
& = \frac{1}{2} \gamma_1^2 \sum_{i=1}^N\sum_{l=1}^K\sum_{j\in C_l}\bigg[\sum_{k=1}^Kq_{jk}^c\left(\alpha_{il}^* - \mu_{ik}\right)^2 \bigg]\\
& = \frac{1}{2} \gamma_1^2 \sum_{l=1}^K\sum_{j\in C_l}\sum_{k=1}^Kq_{jk}^c \| \alpha_{\cdot l}^* - \mu_{\cdot k}\|^2.
\tag{8} \label{eq.s2}
\end{align*}
Let 
$d_{\min} =  \min_{l\neq l'}\|\alpha_{\cdot l}^*-\alpha_{\cdot l'}^*\|$. Note that for any $\mu_{\cdot k}$ there exists at most one $\alpha^*_{\cdot l}$ such that $\|\alpha_{\cdot l}^*-\mu_{\cdot k}\|< d_{\min}/2$. 

We consider two possible cases:
\begin{itemize}
\item [Case 1:] For each $\alpha^*_{\cdot l}$, there exists one and only one $\mu_{\cdot k}$ such that $\|\alpha_{\cdot l}^*-\mu_{\cdot k}\|< d_{\min}/2$. 
\item [Case 2:] There exists some $\alpha^*_{\cdot l}$ such that no $\mu_{\cdot k}$ is within its $d_{\min}/2$-radius.
\end{itemize}
Technically, another possibility is there exists some $\alpha^*_{\cdot l}$ such that at least two  $\mu_{\cdot k}$ are within its $d_{\min}/2$-radius.
But it has been included in Case 2: there must be some $\alpha^*_{\cdot l}$ which does not have any $\mu_{\cdot k}$ nearby since one $\mu_{\cdot k}$ cannot be within  the  $d_{\min}/2$-radius of two $\alpha_{\cdot l}^*$.
We now prove in either Case 1 or Case 2, $\sum_{l=1}^K\sum_{j\in C_l}\sum_{k=1}^Kq_{jk}^c\norm{\alpha_{\cdot l}^* - \mu_{\cdot k}}^2$ can be bounded from below by the misclassification rate up to constants.

The one-to-one correspondence in Case 1 induces a permutation $s$ on $\{1,...,K\}$.  Case 1 implies $\norm{\alpha_{\cdot l}^* - \mu_{\cdot k}}_2^2 \geq d_{\min}^2/4$ for $l \neq s(k)$.
\begin{align*}
& \sum_{l=1}^K\sum_{j\in C_l}\sum_{k=1}^Kq_{jk}^c\norm{\alpha_{\cdot l}^* - \mu_{\cdot k}}^2 \\
\geq \,\, & \sum_{k=1}^K \sum_{l\neq s(k) }\sum_{j\in C_l}q_{jk}^c\norm{\alpha_{\cdot l}^* - \mu_{\cdot k}}^2 \\
\geq \,\, & \frac{d_{\min}^2}{4} \sum_{k=1}^K \sum_{l\neq s(k) }\sum_{j\in C_l} q_{jk}^c \\
= \,\, & \frac{d_{\min}^2}{4} P \left (1- \frac{1}{P}  \sum_{k=1}^K \sum_{j=1}^P q_{jk}^c I^c_{j,s(k)}  \right ).
\end{align*}
In Case 2, let $s$ be any permutation on $\{1,...,K\}$, and $l$ be the class label such that $\|\alpha_{\cdot l}^*-\mu_{\cdot k}\|^2\geq d_{\min}^2/4$ for all $k$. 
\begin{align*}
& \sum_{l=1}^K\sum_{j\in C_l}\sum_{k=1}^Kq_{jk}^c\norm{\alpha_{\cdot l}^* - \mu_{\cdot k}}^2 \\
\geq \,\, & \sum_{j\in C_l}\sum_{k=1}^Kq_{jk}^c\norm{\alpha_{\cdot l}^* - \mu_{\cdot k}}^2 \\
\geq \,\, & \frac{d_{\min}^2}{4} \sum_{j\in C_l}\sum_{k=1}^Kq_{jk}^c  \\ 
\geq \,\, & \frac{d_{\min}^2}{4} \pi_{\min} \frac{1}{K} P \left (1- \frac{1}{P}  \sum_{k=1}^K \sum_{j=1}^P q_{jk}^c I^c_{j,s(k)}  \right ).
\end{align*}

Next we give a lower bound for $d_{\min}$. According to the model assumption that $\balp_i^*\sim N(\bf{0}, \bome^*)$, we have
\begin{align*}
   & \alpha_{ik}^* -\alpha_{il}^*\sim N(0, \omega_{kk}^*+\omega_{ll}^*-2\omega_{kl}^*),\\
   & \sum_{i=1}^N(\alpha_{ik}^* -\alpha_{il}^*)^2 \sim  (\omega_{kk}^*+\omega_{ll}^*-2\omega_{kl}^*)\chi_N^2,
\end{align*}
where $\chi_N^2$ is a chi-squared distribution with $N$ degrees of freedom. Let $A_{kl} = \sum_{i=1}^N(\alpha_{ik}^* -\alpha_{il}^*)^2$ and $B_3 = \min_{1\leq k < l \leq K}(\omega_{kk}^*+\omega_{ll}^*-2\omega_{kl}^*)$. By the sub-exponential tail bounds of chi-squared random variable,
\begin{align*}
    & \mathbb{P}\left( \left .A_{kl} < \frac{N B_3}{2} \right | \V{c}^* \right)\leq
    \mathbb{P}\left( \left . A_{kl} \leq \frac{N(\omega_{kk}^*+\omega_{ll}^*-2\omega_{kl}^*)}{2} \right | \V{c}^*\right) \leq \exp \left (-\frac{N}{16} \right ).
\end{align*}
The last inequality follows the tail bound $P(\chi^2_N-N<N/2)<\exp(-N/16)$. See Theorem 2 in \cite{ghosh2021exponential}. 
Then
\begin{align*}
   & \mathbb{P}\left(\left . d_{\min}^2 > \frac{NB_3}{2}\right | \V{c}^*\right) \\
    = \,\, &  \mathbb{P}\left(\left . \min_{1\leq k<l\leq K}A_{kl} > \frac{NB_3}{2}\right | \V{c}^*\right)\\
   \geq \,\, &  1- \bigcup_{1\leq k<l\leq K}\mathbb{P}\left( \left. A_{kl} < \frac{NB_3}{2}\right | \V{c}^*\right)\\
    \geq \,\, & 1- \frac{K(K-1)}{2}\exp\left (-\frac{N}{16} \right).
\end{align*}
Therefore as $N,P\rightarrow \infty$, 
\begin{align*} 
& \mathbb{P}\left( \left . \bar{J}_{\textnormal{core}}(\bic, \V{\alpha}^*)- \bar{J}_{\textnormal{core}}(\bqc, \V{\mu} )  \geq c_3 NP \min_{\tilde{\V{I}}^c \in \mathcal{E}_{\V{I}^c} } \left (1-\textnormal{Tr} (\M{R}(\bqc,\tilde{\V{I}}^c))\right )  \right | \V{c}^* \right) \rightarrow 1,
\end{align*}
where $c_3$ is a constant. 

By the assumptions on $\mathcal{C}_{\V{\pi}}, \mathcal{C}_{\M{\Omega}},\mathcal{C}_{\V{c}}$, and $\mathcal{C}_{\V{\alpha}}$, 
\begin{align*}
& \sum_{j=1}^P\left(\sum_{k=1}^Kq_{jk}^c\log
   \pi_k\right) = O(P), \\
&  \sum_{i=1}^N\left(-\frac{1}{2}\log|\bome| - \frac{1}{2}\tr{\left(\mathbb{E}_{q_2}(\alpha_i\alpha_i^\top)\bome^{-1}\right)}\right)=O(N), \\
&  - \sum_{j=1}^P\sum_{k=1}^Kq_{jk}^c\log q_{jk}^c = O(P).
\end{align*}
Furthermore,
\begin{align*}
& - \sum_{i=1}^N\int_{\V{\alpha}_i\in \mathbb{R}^{ K}} q_{2i}( \V{\alpha}_i) \log (q_{2i}( \V{\alpha}_i)) d \V{\alpha}_i \\
= \,\, &  \frac{1}{2} \sum_{i=1}^N \log |\M{V}_i|+\frac{NK}{2}(1+\ln(2\pi)) \\
 \leq \,\,  & \frac{1}{2} \sum_{i=1}^N \sum_{k=1}^K \log (v_{ik}^2) + \frac{NK}{2}(1+\ln(2\pi)).
\end{align*}
The inequality holds because the determinant of a correlation matrix is at most 1. Finally,
$- (1/2)\sum_{j=1}^P\sum_{k=1}^K\lambda_j^{*2}q_{jk}^c v^2_{ik}+(1/2)  \sum_{k=1}^K \log (v^2_{ik})$
is maximized at
$v^2_{ik} = \min \left ( \frac{1}{\sum_{j=1}^P\lambda_j^{*2} q_{jk}^c}, B_2 \right )$.
Therefore, 
\begin{align*}
& \quad - \frac{1}{2}\sum_{i=1}^N\sum_{j=1}^P\sum_{k=1}^K\lambda_j^{*2}q_{jk}^c v^2_{ik}+\frac{1}{2} \sum_{i=1}^N \sum_{k=1}^K \log (v^2_{ik})=O(N \log P).
\end{align*}
Putting all terms together, we have proved
\begin{align*} 
& \mathbb{P}\left( \left . \bar{J}_{\textnormal{core}}(\bic, \V{\alpha}^*)+c_1 P+c_2 N\log P \right. \right.\\ 
& \quad \quad \left . \left .-  \bar{J}(q_1(\V{c}), q_2(\V{\alpha}),\V{\pi},\M{\Omega}) \geq c_3 NP \min_{\tilde{\V{I}}^c \in \mathcal{E}_{\V{I}^c} } \left (1-\textnormal{Tr} (\M{R}(\bqc,\tilde{\V{I}}^c))\right ) \right | \V{c}^* \right) \rightarrow 1,
\end{align*}
where $c_1$, $c_2$, and $c_3$ are constants.
\end{proof}

\begin{remark}
The model assumption of HBCM is related to classical cluster analysis based on cluster centroids, such as in k-means. In fact, a key step in the proof of Proposition 5 was analyzing k-means on the true label variables $\V{\alpha}^*$.  However, a direct application of k-means to the data matrix $\M{X}$ does not work due to the variation in $\V{\lambda}$ and $\V{\sigma}^2$.
\end{remark}

\begin{proof}[Proof of Proposition 6]
Let 
\begin{align*}
\hat{J}_1(\bqc, \V{\mu}) & =\sum_{i=1}^N\sum_{j=1}^P\sum_{k=1}^K\left(-\frac{1}{2}\hat{\lambda}_j^2q_{jk}^c \mu_{ik}^2 \right), \\
\bar{J}_1(\bqc, \V{\mu}) & =\sum_{i=1}^N\sum_{j=1}^P\sum_{k=1}^K\left(-\frac{1}{2}\lambda^{*2}_jq_{jk}^c \mu_{ik}^2 \right), \\
\hat{J}_2(\bqc,\V{\mu}) & = \sum_{i=1}^N\sum_{j=1}^P\sum_{k=1}^Kq_{jk}^c\hat{\lambda}_jX_{ij}\mu_{ik}, \\
\bar{J}_2(\bqc,\V{\mu}) & =\sum_{i=1}^N\sum_{j=1}^P\sum_{k=1}^K q_{jk}^c\lambda_j^{*}\mathbb{E}(X_{ij}|\V{\alpha}^*,\V{c}^*)\mu_{ik}. 
\end{align*}
We first prove the convergence of $\hat{J}_1(\bqc, \V{\mu})$. Since
\begin{align*}
    &\left|\hat{J}_1(\bqc, \V{\mu}) -\bar{J}_1(\bqc, \V{\mu})\right|  = \left |\sum_{i=1}^N\sum_{j=1}^P\sum_{k=1}^K q_{jk}^c\mu_{ik}^2\left(\hat{\lambda}_j^2 -\lambda_j^{*2}\right) \right|\leq NKB_1^2\sum_{j=1}^P|\hat{\lambda}_j -\lambda_j^{*}| |\hat{\lambda}_j +\lambda_j^{*}|,
\end{align*}
by Proposition 3,
\begin{align*}
    \mathbb{P}\left( \max_{q_1(\V{c})\in \mathcal{C}_{\V{c}}, q_2(\V{\alpha}) \in \mathcal{C}_{\V{\alpha}}}  \left|\hat{J}_{1}(\bqc, \V{\mu}) - \bar{J}_{1}(\bqc, \V{\mu}) \right|\geq NP\epsilon  \middle| \V{c}^* \right) \rightarrow 0.
\end{align*}
Next we prove the convergence of $\hat{J}_2(\bqc, \V{\mu})$. Write
\begin{align*}
    &\left|\hat{J}_2(\bqc, \V{\mu}) -\bar{J}_2(\bqc, \V{\mu})\right| \\
     = \,\, & \left| \sum_{i=1}^N\sum_{j=1}^P\sum_{k=1}^K\left(\hat{\lambda}_jq_{jk}^c\mu_{ik}X_{ij}-\lambda_j^{*}q_{jk}^c\mu_{ik}\mathbb{E}(X_{ij}|\balp^*,\V{c}^*)\right) \right| \\
   \leq \,\, &  \sum_{k=1}^K \left| \sum_{i=1}^N\sum_{j=1}^P\left(\hat{\lambda}_jX_{ij} -\lambda_j^{*}\mathbb{E}(X_{ij}|\balp^*,\V{c}^*) \right)q_{jk}^c\mu_{ik} \right|\\
    = \,\, &  \sum_{k=1}^K \bigg| \sum_{i=1}^N\sum_{j=1}^P\bigg(\lambda_j^*(X_{ij} - \mathbb{E}(X_{ij}|\balp^*,\V{c}^*)) +(\hat{\lambda}_j-\lambda_j^*)(X_{ij} - \mathbb{E}(X_{ij}|\balp^*,\V{c}^*)) \\
    & + (\hat{\lambda}_j-\lambda_j^*) \mathbb{E}(X_{ij}|\balp^*,\V{c}^*) \bigg)q_{jk}^c\mu_{ik} \bigg|\\
     \leq \,\, &  \underbrace{\sum_{k=1}^K \left| \sum_{i=1}^N\sum_{j=1}^P\left(\lambda_j^*(X_{ij} - \mathbb{E}(X_{ij}|\balp^*,\V{c}^*))\right)q_{jk}^c\mu_{ik} \right|}_{(a)} +\\
    &  \underbrace{\sum_{k=1}^K \left| \sum_{i=1}^N\sum_{j=1}^P\left((\hat{\lambda}_j-\lambda_j^*)(X_{ij} - \mathbb{E}(X_{ij}|\balp^*,\V{c}^*))\right)q_{jk}^c\mu_{ik} \right|}_{(b)} + \\
    &  \underbrace{\sum_{k=1}^K \left| \sum_{i=1}^N\sum_{j=1}^P\left((\hat{\lambda}_j-\lambda_j^*) \mathbb{E}(X_{ij}|\balp^*,\V{c}^*) \right)q_{jk}^c\mu_{ik} \right|}_{(c)}.
\end{align*}
We first prove the following statement that
\begin{align*}
 & \underset{q_1\in \mathcal{C}_{\V{c}}, q_2 \in \mathcal{C}_{\V{\alpha}}}{\max} \left| \sum_{i=1}^N\sum_{j=1}^P\left(X_{ij} -\mathbb{E}(X_{ij}|\balp^*,\V{c}^*) \right)q_{jk}^c\mu_{ik}  \right|\\
   = \,\, & \underset{q^c_{jk}\in \{0,1\}, \mu_{ik} \in \{-B_1, B_1\}}{\max} \left| \sum_{i=1}^N\sum_{j=1}^P\left(X_{ij} -\mathbb{E}(X_{ij}|\balp^*,\V{c}^*) \right)q_{jk}^c\mu_{ik} \right|.
\end{align*}
Let $f(\V{u},\V{v}) = \sum_{i=1}^n\sum_{j=1}^p(X_{ij}- E_{ij})u_iv_j$, where $E_{ij}=\mathbb{E}(X_{ij}|\balp^*,\V{c}^*), u_i\in [0,1], v_j \in [-B_1, B_1]$. If the maximum occurs at $(\V{u},\V{v})$ with $v_1\in(-B_1, B_1)$, consider $\V{v}^{'},\V{v}^{''}$ with $v_1^{'}=-B_1,v_1^{''}=B_1$ and the other entries are identical to that of $\V{v}$. Then we can write $f(\V{u},\V{v}) =\frac{B_1-v_1}{2B_1}f(\V{u},\V{v}^{'})+\frac{v_1+B_1}{2B_1}f(\V{u},\V{v}^{''})$. Since $f(\V{u},\V{v})$ is a linear function of $v_1$, the value of $f(\V{u},\V{v})$ is between $f(\V{u},\V{v}^{'})$ and $f(\V{u},\V{v}^{''})$, which implies that both $(\V{u},\V{v}^{'})$ and $(\V{u},\V{v}^{''})$ are maximizers and we take either $(\V{u},\V{v}^{'})$ or $(\V{u},\V{v}^{''})$ instead of $(\V{u},\V{v})$. If there are other entries of $\V{v}$ strictly between $(-B_1,B_1)$, we can use the same argument until we find a maximizer with all entries equal to $-B_1$ or $B_1$. The same argument applies to $\V{u}$ being an maximizer with all entries equal to 0 or 1. Next we provide the upper bounds of the three components (a), (b) and (c).

The first term (a) satisfies
\begin{align*}
    & \mathbb{P}\left( \left. \max_{q_1\in \mathcal{C}_{\V{c}}, q_2 \in \mathcal{C}_{\V{\alpha}}}
    \sum_{k=1}^K \left| \sum_{i=1}^N\sum_{j=1}^P\left(\lambda_j^*X_{ij} -\lambda_j^*\mathbb{E}(X_{ij}|\balp^*,\V{c}^*) \right)q_{jk}^c\mu_{ik}  \right|\geq NP\epsilon \right |\V{c}^* \right)\\
    = \,\, & \mathbb{P}\left ( \left. \max_{q^c_{jk}\in \{0,1\}, \mu_{ik} \in \{-B_1, B_1\}}
    \sum_{k=1}^K \left| \sum_{i=1}^N\sum_{j=1}^P\left(\lambda_j^*X_{ij} -\lambda_j^*\mathbb{E}(X_{ij}|\balp^*,\V{c}^*) \right)q_{jk}^c\mu_{ik}  \right|\geq NP\epsilon \right |\V{c}^* \right)\\
    \leq \,\, & \sum_{k=1}^K\mathbb{P}\left(\left.\max_{q^c_{jk}\in \{0,1\}, \mu_{ik} \in \{-B_1, B_1\}}
    \left| \sum_{i=1}^N\sum_{j=1}^P\left(  \lambda_j^*X_{ij} -\lambda_j^*\mathbb{E}(X_{ij}|\balp^*,\V{c}^*) \right)q_{jk}^c\mu_{ik} \right|\geq NP\frac{\epsilon}{K} \right |\V{c}^* \right)\\
    \leq \,\, & K2^{N+P} 2\exp \left (-\frac{NP^2\epsilon^2}{2(B_1K)^2\sum_{j=1}^P\lambda_j^{*2}\sigma_j^{*2}} \right ).
\end{align*}
The second term (b) satisfies
{\small
\begin{align*}
     & \mathbb{P}\left(\left.\max_{q_1\in \mathcal{C}_{\V{c}}, q_2 \in \mathcal{C}_{\V{\alpha}}}
    \sum_{k=1}^K \left| \sum_{i=1}^N\sum_{j=1}^P\left((\hat{\lambda}_j-\lambda_j^*)(X_{ij} - \mathbb{E}(X_{ij}|\balp^*,\V{c}^*))\right)q_{jk}^c\mu_{ik} \right|
  \geq NP\epsilon \right |\V{c}^* \right)\\   
  = \,\, & \mathbb{P}\left( \left. \max_{q_1\in \mathcal{C}_{\V{c}}, q_2 \in \mathcal{C}_{\V{\alpha}}}
    \sum_{k=1}^K \left| \sum_{i=1}^N\sum_{j=1}^P\left((\hat{\lambda}_j-\lambda_j^*)(X_{ij} - \mathbb{E}(X_{ij}|\balp^*,\V{c}^*))\right)q_{jk}^c\mu_{ik} \right|
  \geq NP\epsilon, \exists j, \left|\hat{\lambda}_j -\lambda_j^{*}\right| \geq \epsilon_1 \right |\V{c}^* \right) \\
  & + \mathbb{P}\left(\left. \max_{q_1\in \mathcal{C}_{\V{c}}, q_2 \in \mathcal{C}_{\V{\alpha}}}
    \sum_{k=1}^K \left| \sum_{i=1}^N\sum_{j=1}^P\left((\hat{\lambda}_j-\lambda_j^*)(X_{ij} - \mathbb{E}(X_{ij}|\balp^*,\V{c}^*))\right)q_{jk}^c\mu_{ik} \right|
  \geq NP\epsilon,   \forall j, \left|\hat{\lambda}_j -\lambda_j^{*}\right| < \epsilon_1 \right |\V{c}^* \right)\\
  \leq \,\, & \sum_{j=1}^P\mathbb{P}\left(\left. \left|\hat{\lambda}_j -\lambda_j^{*}\right| \geq \epsilon_1 \right |\V{c}^* \right) +  \mathbb{P}\left( \left. \max_{q^c_{jk}\in \{0,1\}, \mu_{ik} \in \{-B_1, B_1\}}
    \sum_{k=1}^K \left| \sum_{i=1}^N\sum_{j=1}^P\left(X_{ij} - \mathbb{E}(X_{ij}|\balp^*,\V{c}^*)\right)q_{jk}^c\mu_{ik} \right|\epsilon_1
  \geq NP\epsilon \right |\V{c}^* \right)\\
  \leq \,\, & \sum_{j=1}^P\mathbb{P}\left( \left. \left|\hat{\lambda}_j -\lambda_j^{*}\right| \geq \epsilon_1 \right |\V{c}^* \right) +   \sum_{k=1}^K\mathbb{P}\left( \left.\max_{q^c_{jk}\in \{0,1\}, \mu_{ik} \in \{-B_1, B_1\}}
    \left| \sum_{i=1}^N\sum_{j=1}^P\left(X_{ij} - \mathbb{E}(X_{ij}|\balp^*,\V{c}^*)\right)q_{jk}^c\mu_{ik} \right|
  \geq \frac{NP\epsilon}{K\epsilon_1} \right |\V{c}^* \right)\\
  \leq \,\, &  P^2C_1\exp(-C_2N \epsilon_1^2) + K2^{N+P} 2\exp \left(-\frac{NP^2\epsilon^2}{2(B_1K)^2\epsilon_1^2\sum_{j=1}^P\sigma_j^{*^2}} \right).
\end{align*}
}%

The third term (c) satisfies 
{\small
\begin{align*}
    & \mathbb{P}\left(\left. \max_{q_1\in \mathcal{C}_{\V{c}}, q_2 \in \mathcal{C}_{\V{\alpha}}}
    \sum_{k=1}^K \left| \sum_{i=1}^N\sum_{j=1}^P\left((\hat{\lambda}_j-\lambda_j^*) \mathbb{E}(X_{ij}|\balp^*,\V{c}^*) \right)q_{jk}^c\mu_{ik} \right|
  \geq NP\epsilon \right |\V{c}^* \right)\\
  \leq \,\, & \sum_{j=1}^P\mathbb{P}\left( \left. \left|\hat{\lambda}_j -\lambda_j^{*}\right| \geq \epsilon_1 \right |\V{c}^* \right)+ \mathbb{P}\left( \left. \max_{q^c_{jk}\in \{0,1\}, \mu_{ik} \in \{-B_1, B_1\}}
    \sum_{k=1}^K \left| \sum_{i=1}^N\sum_{j=1}^P\mathbb{E}(X_{ij}|\balp^*,\V{c}^*)q_{jk}^c\mu_{ik} \right|\epsilon_1
  \geq NP\epsilon \right |\V{c}^* \right)\\
  \leq  \,\, &  \sum_{j=1}^P\mathbb{P}\left( \left. \left|\hat{\lambda}_j -\lambda_j^{*}\right| \geq \epsilon_1 \right |\V{c}^* \right)+2^N \sum_{k=1}^K \mathbb{P}\left( \left. \max_{q^c_{jk}\in \{0,1\}}
    \left|\sum_{i=1}^N\alpha_{ik}^*\right| \left| \sum_{j=1}^P \lambda_j^*q_{jk}^c \right|
  \geq \frac{NP\epsilon}{B_1\epsilon_1} \right |\V{c}^* \right)\\
  \leq\,\,  &   \sum_{j=1}^P\mathbb{P}\left( \left. \left|\hat{\lambda}_j -\lambda_j^{*}\right| \geq \epsilon_1 \right |\V{c}^* \right) + 
  2^N\sum_{k=1}^K\mathbb{P}\left( \left.
   \left|\sum_{i=1}^N\alpha_{ik}^*\right| P\gamma_2
  \geq \frac{NP\epsilon}{B_1\epsilon_1} \right |\V{c}^* \right)\\
 \leq \,\, & 
  P^2C_1\exp(-C_2N \epsilon_1^2) + K2^N2\exp\left(-\frac{N\epsilon^2}{2B_1^2\gamma_2^2\epsilon_1^2 \max_{k} \omega_{kk}}\right).
\end{align*}
}
Putting everything together,
\begin{align*}
 & \mathbb{P}\left( \max_{q_1\in \mathcal{C}_{\V{c}}, q_2 \in \mathcal{C}_{\V{\alpha}}}  \left|\hat{J}_{\textnormal{core}}(\bqc, \V{\mu}) - \bar{J}_{\textnormal{core}}(\bqc, \V{\mu}) \right|\geq NP\epsilon  \middle| \V{c}^* \right)\\   
  \leq \,\, & \mathbb{P}\left(\max_{q_1\in \mathcal{C}_{\V{c}}, q_2 \in \mathcal{C}_{\V{\alpha}}}
    \sum_{k=1}^K \left| \sum_{i=1}^N\sum_{j=1}^P\left(\lambda_j^*X_{ij} -\lambda_j^*\mathbb{E}(X_{ij}|\balp^*,\V{c}^*) \right)q_{jk}^c\mu_{ik}  \right|\geq NP\epsilon/3  \middle|\V{c}^* \right)\\
    & + \mathbb{P}\left(\max_{q_1\in \mathcal{C}_{\V{c}}, q_2 \in \mathcal{C}_{\V{\alpha}}}
    \sum_{k=1}^K \left| \sum_{i=1}^N\sum_{j=1}^P\left((\hat{\lambda}_j-\lambda_j^*)(X_{ij} - \mathbb{E}(X_{ij}|\balp^*,\V{c}^*))\right)q_{jk}^c\mu_{ik} \right|
  \geq NP\epsilon/3  \middle|\V{c}^* \right)\\
  & + \mathbb{P}\left(\max_{q_1\in \mathcal{C}_{\V{c}}, q_2 \in \mathcal{C}_{\V{\alpha}}}
    \sum_{k=1}^K \left| \sum_{i=1}^N\sum_{j=1}^P\left((\hat{\lambda}_j-\lambda_j^*) \mathbb{E}(X_{ij}|\balp^*,\V{c}^*) \right)q_{jk}^c\mu_{ik} \right|
  \geq NP\epsilon/3  \middle|\V{c}^* \right)\\
  \leq \,\, & K2^{N+P} 2\exp\left (-\frac{NP^2\epsilon^2}{18(B_1K)^2\sum_{j=1}^P\lambda_j^{*2}\sigma_j^{*2}}\right ) \\
  &  + P^2C_1\exp(-C_2N\epsilon_1^2/9) + K2^{N+P} 2\exp \left(-\frac{NP^2\epsilon^2}{18(B_1K)^2\epsilon_1^2\sum_{j=1}^P\sigma_j^{*^2}} \right) \\
  &  + P^2C_1\exp(-C_2N\epsilon_1^2/9) + K2^N2\exp\left(-\frac{N\epsilon^2}{18B_1^2\gamma_2^2\epsilon_1^2 \max_k \omega_{kk}}\right).
\end{align*}
The theorem holds by letting $\epsilon_1^2  = \frac{\epsilon^2}{36 B_1^2\gamma^2 \max_k \omega_{kk}\ln2}$.
\end{proof}
\begin{proof}[Proof of Theorem 2]
Recall that  $\hat{\V{q}}^{\V{c}}$ is the maximizer of $\hat{J}(q_1(\V{c}), q_2(\V{\alpha}),\V{\pi},\M{\Omega})$ over $q_1(\V{c})$.
More specifically, denote 
\begin{align*}
(\hat{\V{q}}^{\V{c}},\hat{q}_2(\V{\alpha}),\hat{\V{\pi}},\hat{\V{\Omega}})=\max_{q_1(\V{c})\in \mathcal{C}_{\V{c}},q_2(\V{\alpha})\in \mathcal{C}_{\V{\alpha}},\V{\pi}\in \mathcal{C}_{\V{\pi}},\V{\Omega}\in \mathcal{C}_{\V{\Omega}}} \hat{J}(q_1(\V{c}), q_2(\V{\alpha}),\V{\pi},\M{\Omega}).
\end{align*}
Let $\check{q}_{2i}(\V{\alpha}_i)$ be the density function of  $N(\V{\alpha}^*_i,\check{\M{V}})$ where $\check{\M{V}}=\textnormal{diag}(1/P,...,1/P)$, and let $\check{q}_2(\V{\alpha})=\prod_{i=1}^N \check{q}_{2i}(\V{\alpha}_i)$. Note that 
\begin{align*}
- \sum_{i=1}^N\int_{\V{\alpha}_i\in \mathbb{R}^{ K}} \check{q}_{2i}( \V{\alpha}_i) \log (\check{q}_{2i}( \V{\alpha}_i)) d \V{\alpha}_i =O(N \log P).
\end{align*}
For all $\epsilon>0$,
\begin{align*}
& \mathbb{P} \left( \left. \min_{\tilde{\V{I}}^c \in \mathcal{E}_{\V{I}^c} } \left (1-\textnormal{Tr} (\M{R}(\hat{\V{q}}^c,\tilde{\V{I}}^c))\right )\geq \epsilon  \right |\V{c}^* \right) \\
= \,\, & \mathbb{P} \left( \left. \min_{\tilde{\V{I}}^c \in \mathcal{E}_{\V{I}^c} } \left (1-\textnormal{Tr} (\M{R}(\hat{\V{q}}^c,\tilde{\V{I}}^c))\right )\geq \epsilon, \bar{J}_{\textnormal{core}}(\bic,  \V{\alpha}^*)+c_1 P+c_2 N\log P \right. \right.\\ 
& \quad \quad \left . \left .-  \bar{J}(\hat{\V{q}}^{\V{c}}, \hat{q}_2(\V{\alpha}),\hat{\V{\pi}},\hat{\M{\Omega}}) \geq  c_3 NP \min_{\tilde{\V{I}}^c \in \mathcal{E}_{\V{I}^c} } \left (1-\textnormal{Tr} (\M{R}(\bqc,\tilde{\V{I}}^c))\right ) \right |\V{c}^* \right)\\
 \,\, & +\mathbb{P} \left( \left. \min_{\tilde{\V{I}}^c \in \mathcal{E}_{\V{I}^c} } \left (1-\textnormal{Tr} (\M{R}(\hat{\V{q}}^c,\tilde{\V{I}}^c))\right )\geq \epsilon, \bar{J}_{\textnormal{core}}(\bic,  \V{\alpha}^*)+c_1 P+c_2 N\log P \right. \right.\\ 
& \quad \quad \left . \left .-  \bar{J}(\hat{\V{q}}^{\V{c}}, \hat{q}_2(\V{\alpha}),\hat{\V{\pi}},\hat{\M{\Omega}}) <  c_3 NP \min_{\tilde{\V{I}}^c \in \mathcal{E}_{\V{I}^c} } \left (1-\textnormal{Tr} (\M{R}(\bqc,\tilde{\V{I}}^c))\right ) \right |\V{c}^* \right) \\
\leq \,\, & \mathbb{P} \left( \left.  \bar{J}_{\textnormal{core}}(\bic,  \V{\alpha}^*)+c_1 P+c_2 N\log P -  \bar{J}(\hat{\V{q}}^{\V{c}}, \hat{q}_2(\V{\alpha}),\hat{\V{\pi}},\hat{\M{\Omega}}) \geq c_3 NP \epsilon  \right |\V{c}^* \right) +o(1) \\
\leq \,\, & \mathbb{P} \left( \left.  \bar{J}_{\textnormal{core}}(\bic,  \V{\alpha}^*)+c_1 P+c_2 N\log P - \hat{J}(\V{I}^c, \check{q}_2(\V{\alpha}),\hat{\V{\pi}},\hat{\M{\Omega}}) \right. \right. \\
& \quad \quad \left . \left .+\hat{J}(\hat{\V{q}}^{\V{c}}, \hat{q}_2(\V{\alpha}),\hat{\V{\pi}},\hat{\M{\Omega}}) -\bar{J}(\hat{\V{q}}^{\V{c}}, \hat{q}_2(\V{\alpha}),\hat{\V{\pi}},\hat{\M{\Omega}}) \geq c_3 NP \epsilon  \right |\V{c}^* \right) +o(1) \\
 \leq \,\, & \mathbb{P} \left( \left.  \bar{J}_{\textnormal{core}}(\bic,  \V{\alpha}^*) -\hat{J}_{\textnormal{core}}(\bic,  \V{\alpha}^*) \geq \frac{1}{2 }c_3 NP \epsilon -\tilde{c}_1 P-\tilde{c}_2 N\log P \right |\V{c}^* \right)  \\
& \quad  +\mathbb{P}\left ( \left. \hat{J}(\hat{\V{q}}^{\V{c}}, \hat{q}_2(\V{\alpha}),\hat{\V{\pi}},\hat{\M{\Omega}}) -\bar{J}(\hat{\V{q}}^{\V{c}}, \hat{q}_2(\V{\alpha}),\hat{\V{\pi}},\hat{\M{\Omega}}) \geq  \frac{1}{2 }c_3 NP  \epsilon  \right |\V{c}^* \right) +o(1) \rightarrow 0,
\end{align*}
where $c_1,c_2,c_3,\tilde{c}_1$, and $\tilde{c}_2$ are constants.
\end{proof}



 \newcommand{\noop}[1]{}

\end{document}